\documentclass[lettersize,journal]{IEEEtran}
\usepackage{amsmath,amsfonts}
\usepackage{algorithmic}
\usepackage{algorithm}
\usepackage{multirow}
\usepackage{amsmath, amsthm}
\usepackage{array}
\usepackage{textcomp}
\usepackage{stfloats}
\usepackage{subcaption}
\usepackage{url}
\usepackage{booktabs}
\usepackage{verbatim}
\usepackage{graphicx}
\usepackage{cite}
\newtheorem{theorem}{Theorem}

\newcommand{\argmax}{\operatorname*{arg\,max}}

\makeatletter
\def\@maketitle{%
  \newpage
  \null
  \vspace*{-40pt} 
  \begin{center}
    \footnotesize\itshape
    This work has been submitted to the IEEE for possible publication.\\ Copyright may be transferred without notice, after which this version may no longer be accessible.
  \end{center}
  \vspace{12pt} 
  \begin{center}%
    {\LARGE \bfseries \@title \par}%
    \vskip 1.5em%
    {\large
      \lineskip .5em%
      \begin{tabular}[t]{c}%
        \@author
      \end{tabular}\par}%
  \end{center}}
\makeatother

\begin{document} 

\title{Embedded Mean Field Reinforcement Learning for Perimeter-defense Game}

\author{Li Wang, Xin Yu, Xuxin Lv, Gangzheng Ai, Wenjun Wu
\thanks{Manuscript received May 20, 2025; revised xxx. This work was supported by the National Science and Technology Major Project (No. 2022ZD0117402). (Corresponding author: Wenjun Wu.)}
\thanks{Li Wang, Xuxin Lv, Wenjun Wu are with School of Artificial Intelligence, Beihang University, Beijing, China, and also with Hangzhou International Innovation Institute, Beihang University, Hangzhou, China (e-mail: wangli\_42@buaa.edu.cn, lvxuxin@buaa.edu.cn, wwj09315@buaa.edu.cn).}
\thanks{Xin Yu is with School of Computer Science and Engineering, Beihang University, Beijing, China (e-mail: nlsdeyuxin@buaa.edu.cn).}
\thanks{Gangzheng Ai is with the School of Astronautics, Beihang University, Beijing, China (e-mail: 1328738293@qq.com).}
\thanks{
\textit{This work has been submitted to the IEEE for possible publication. 
Copyright may be transferred without notice, after which this version may no longer be accessible.}
}
}


\maketitle

\begin{abstract}
With the rapid advancement of unmanned aerial vehicles (UAVs) and missile technologies, perimeter-defense game between attackers and defenders for the protection of critical regions have become increasingly complex and strategically significant across a wide range of domains. However, existing studies predominantly focus on small-scale, simplified two-dimensional scenarios, often overlooking realistic environmental perturbations, motion dynamics, and inherent heterogeneity—factors that pose substantial challenges to real-world applicability. To bridge this gap, we investigate large-scale heterogeneous perimeter-defense game in a three-dimensional setting, incorporating realistic elements such as motion dynamics and wind fields. We derive the Nash equilibrium strategies for both attackers and defenders, characterize the victory regions, and validate our theoretical findings through extensive simulations. To tackle large-scale heterogeneous control challenges in defense strategies, we propose an Embedded Mean-Field Actor-Critic (EMFAC) framework. EMFAC leverages representation learning to enable high-level action aggregation in a mean-field manner, supporting scalable coordination among defenders. Furthermore, we introduce a lightweight agent-level attention mechanism based on reward representation, which selectively filters observations and mean-field information to enhance decision-making efficiency and accelerate convergence in large-scale tasks. Extensive simulations across varying scales demonstrate the effectiveness and adaptability of EMFAC, which outperforms established baselines in both convergence speed and overall performance. To further validate practicality, we test EMFAC in small-scale real-world experiments and conduct detailed analyses, offering deeper insights into the framework’s effectiveness in complex scenarios.
\end{abstract}

\begin{IEEEkeywords}
Perimeter-defense Game,  Large-scale Heterogeneous, Mean Field Reinforcement Learning, Agent-level Attention, Representation Learning.
\end{IEEEkeywords}

\section{Introduction}
\IEEEPARstart{P}{erimeter} game\cite{ref1, ref2} represent a specialized class of pursuit-evasion problems. As illustrated in Fig.\ref{fig:game_setup}, the environment is characterized by arbitrarily shaped boundaries, such as hemispherical boundaries. In this setting, attackers seek to reach the boundary, while defenders are tasked with intercepting and preventing them from doing so. In practical scenarios, such boundaries often correspond to critical regions that require protection, such as military bases, command centers, naval vessels, or key infrastructure facilities. As an abstract theoretical problem, perimeter-defense game have broad applicability to a variety of real-world domains, including strategic facility defense, missile interception systems\cite{ref3, ref4}, UAV swarm engagements~\cite{ref5}, and patrolling tasks~\cite{ref6}, all of which are of significant relevance in real-world security and defense operations.

\begin{figure}[!t]
    \centering
    \includegraphics[width=2.5in, height=5cm]{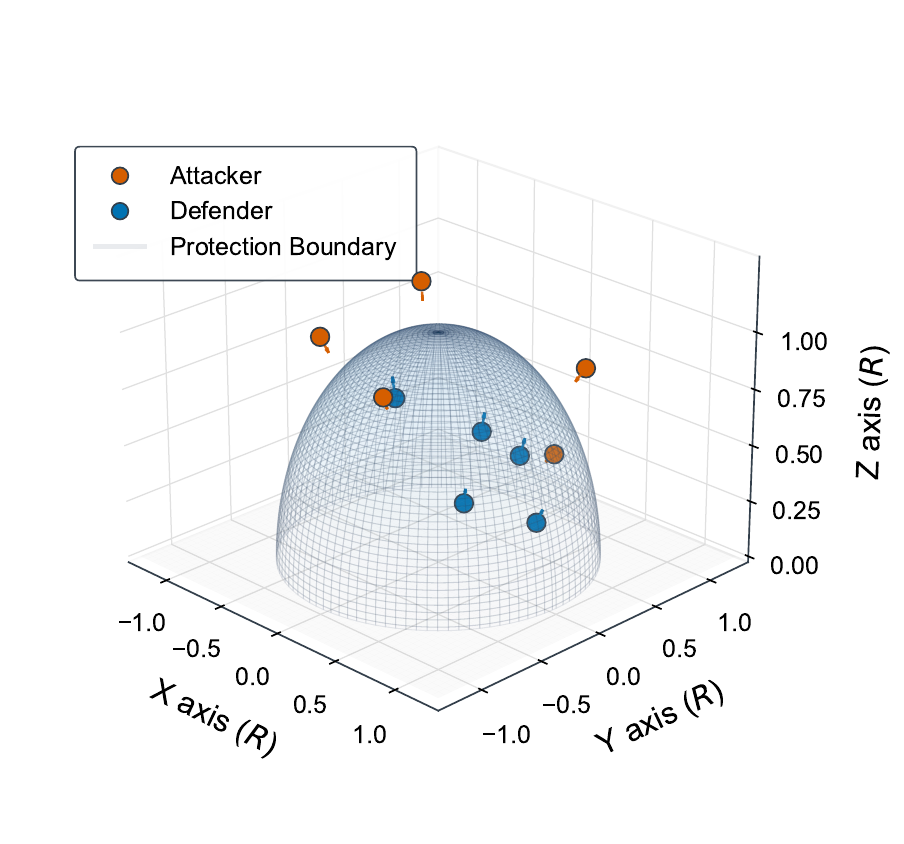}
    \caption{Hemispherical defense scenario with kinematic agents. Attacker $A$ (red) attempts to breach perimeter $\partial\mathcal{H}_R$ while defender $D$ (blue) intercepts.}
    \label{fig:game_setup}
\end{figure}

Various scenarios have been explored in the literature, offering insights into different engagement strategies and motion dynamics \cite{ref1} in perimeter-defense game. \IEEEpubidadjcol Note that we use the term motion dynamics refers specifically to the agents’ control models, such as velocity-controlled or acceleration-controlled dynamics in this paper. Early research mainly focused on global observations scenarios, Bajaj et al. \cite{ref7} investigated a one-dimensional scenario, while Shishika and Kumar \cite{ref2, ref8} extended the analysis to two-dimensional settings and proposed strategy synthesis methods based on problem decomposition and structural property analysis. In multi-agent scenarios, optimal strategies were derived based on task allocation mechanisms \cite{ref11}, and were further generalized to arbitrary convex boundaries \cite{ref12} and conical perimeter \cite{ref15}. As a further extension to three-dimensional scenarios, many studies have been conducted under settings where attackers are confined to the ground and defenders are restricted to the perimeter. Lee et al. \cite{ref9, ref10} analyzed Nash equilibrium strategies for both attackers and defenders, their subsequent work \cite{ref13} addressed the critical challenge posed by limited local observations, proposing a control framework that integrates graph neural networks with communication mechanisms and was substantiated through real-world experiments \cite{ref14}. Additionally, Adler et al. \cite{ref16} examined the influence of agent heterogeneity on overall system performance, focusing on defenders with varying speed profiles. Despite these advancements, existing research remains largely confined to either two-dimensional settings, or imposes significant constraints on the motion space and fails to model the complexities of real-world dynamics, including intricate motion characteristics of agents, variations in control parameters \cite{ref17, ref18}, and heterogeneity in their underlying dynamics \cite{ref19, ref20}. Additionally, environmental perturbations \cite{ref21, ref22} are often overlooked. The enhanced maneuverability of UAVs introduces unprecedented challenges to traditional control methods in perimeter-defense tasks, demanding more agile and adaptive strategies. Meanwhile, the complex interactions inherent in large-scale multi-agent systems significantly exacerbate the difficulty of real-time decision-making. 

To address the intricate challenges posed by complex perimeter-defense game such as environmental perturbations, heterogeneous motion dynamics, and large-scale interactions, multi-agent reinforcement learning (MARL) offers a viable and scalable solution. MARL has demonstrated remarkable progress across various domains, including multi-robot coordination \cite{ref23}, autonomous driving \cite{ref24} and pursuit-evasion scenarios \cite{ref125, ref126, ref127}, proving its effectiveness in handling complex interaction and control problems. A number of studies in MARL have been devoted to addressing the challenges of large-scale and heterogeneous agent systems—challenges that are also relevant in perimeter-defense game. To address the heterogeneity, existing Heterogeneous-Agent (HA) algorithms such as HADDPG and HATD3 \cite{ref25, ref26} leverage the multi-agent advantage decomposition theorem and sequential decision-making to ensure monotonic improvement while imposing fewer constraints on agent composition. To address the challenges of large-scale multi-agent systems—such as the exponential growth of observation and action spaces and increased interaction complexity—Yang et al. \cite{ref27} proposed the Mean-Field Reinforcement Learning (MFRL) framework. By approximating the influence of neighboring agents through mean-field action aggregation, MFRL significantly reduces the complexity of the action space. It has demonstrated strong scalability in domains such as resource allocation \cite{ref28, ref29} and UAV control \cite{ref30}. In more complex large-scale heterogeneous scenarios, Subramanian et al. \cite{ref31} proposed the Multi-Type Mean-Field Reinforcement Learning (MT-MFRL) framework, which extends MFRL by categorizing agents into distinct types and applying mean-field aggregation within each category. Yu \cite{ref32} introduced Hierarchical MFRL for large-scale multi-agent systems, employing multi-level mean-field techniques for improved scalability. Furthermore, Wu et al. \cite{ref33} proposed a reward attribution decomposition approach that incorporates attention mechanisms \cite{ref34} to refine Q-value aggregation within the mean-field framework. Despite these advancements, existing mean-field approaches struggle with action space explosion as agent heterogeneity increases. Moreover, in continuous control tasks, simple action aggregation, even with attention mechanisms, may fail to accurately capture intricate agent interactions, limiting control precision. To extracting meaningful features from high-dimensional features, representation learning \cite{ref35} offers a promising approach for handling large-scale heterogeneous systems. By leveraging state \cite{ref36, ref37, ref38}, action \cite{ref39, ref40}, and reward representations \cite{ref41, ref42}, representation learning enhances the modeling of complex agent interactions and improves learning efficiency. However, its potential in large-scale MARL has yet to be fully realized. 

In this work, we investigate large-scale heterogeneous perimeter-defense game that incorporate realistic factors, including intricate heterogeneous motion dynamics and stochastic environmental perturbations. To address the challenges inherent in such scenarios, we derive Nash equilibrium strategies and propose an embedded mean-field reinforcement learning framework that effectively handles diverse dynamics and incorporates agent-level attention to enhance decision-making in complex interactions. Specifically, we introduce a high-level action learning module that learns high-level action representations through state prediction. To manage the exponentially growing observation space and intricate agent interactions in large-scale systems while ensuring real-time inference and scalability, we propose a lightweight, agent-level attention module based on reward representation. Our approach maintains the original agent architecture, ensuring consistent inference time with minimal additional training overhead. In summary, our key contributions are as follows:
\begin{itemize}
\item \textbf{Formulation of large-scale heterogeneous perimeter-defense defense}: We develop a high-fidelity engagement model that incorporates realistic motion dynamics and wind perturbations, reducing simplifying assumptions to enhance real-world applicability.

\item \textbf{Derivation of Nash equilibrium strategies}:
We derive Nash equilibrium strategies for both attackers and defenders, providing a rigorous theoretical foundation for analyzing large-scale strategic interactions.

\item \textbf{Embedded mean field reinforcement learning}: To handle large-scale heterogeneous tasks, we introduce a novel framework that utilizes representation learning to extract high-level action embeddings for mean-field aggregation. Additionally, agent-level attention mechanisms enable agents to prioritize critical information in dynamic multi-agent scenarios, improving decision-making efficiency with low overhead.

\item \textbf{Comprehensive validation across simulation and real-world experiments}: We conduct extensive simulations to verify the Nash equilibrium strategies and analyze the characterization of winning region. Besides, our framework consistently outperforms baselines across varying scales. To bridge the gap between simulation and reality, we further validate our approach through real-world UAV experiments.
\end{itemize}

The remainder of this article is organized as follows. Section \ref{sec:derivate_of_nash} introduces the derivation and properties of the Nash equilibrium in the three-dimensional perimeter-defense game, serving as the theoretical foundation for MARL-based solutions. Section \ref{sec:marl_for_pg} describes the system model for applying MARL to solve the large-scale heterogeneous perimeter-defense game. Section \ref{sec:emfac_method} presents the proposed EMFAC method in detail. Section \ref{sec:experiment} reports both simulation and real-world experimental results. Finally, Section \ref{sec:conclusion} concludes the paper.

\section{Derivation of Nash Equilibrium Strategies}
\label{sec:derivate_of_nash}
As illustrated in Fig.~\ref{fig:game_setup}, we consider a large-scale heterogeneous perimeter-defense game with an equal number of attackers and defenders, played along a hemispherical boundary. Unlike previous studies \cite{ref9}, both attackers and defenders are allowed to maneuver freely in three-dimensional space in our discussion. Our formulation further incorporates realistic motion dynamics and wind field perturbations, with a focus on enabling defenders to learn effective interception strategies. To tackle this challenging problem, we first analyze a simplified one-on-one case in this section, deriving the Nash equilibrium strategies for both attacker and defender. Building on the strategies, we extend our approach to the more complex many-on-many setting and leverage MARL to control the defenders effectively in section \ref{sec:marl_for_pg}.
\subsection{Problem Assumptions}
 We first analyze a simplified one-on-one case without perturbations. Specifically, we analyze the interaction between two agents operating on a hemispherical surface of radius $R$, denoted as $\partial\mathcal{H}_R$ (see Table~\ref{tab:notation} for key parameters). Both agents follow first-order kinematics: the defender (\( D \)) starts inside \( \mathcal{H}_R \) and moves with unit speed, aiming to intercept the attacker (\( A \)) before it reaches the boundary \( \partial\mathcal{H}_R \); the attacker, in contrast, begins outside \( \mathcal{H}_R \) and travels at a speed \( v \leq 1 \), seeking to reach any point \( Z_B \) on the hemispherical boundary. The attacker selects a point \( B \in \partial\mathcal{H}_R \) as the breach point, which is observable by the defender. The game terminates when either:
\begin{itemize}
    \item Successful Breach: 
    \[
    \inf_{z\in\partial\mathcal{H}_R}\|Z_A(T) - z\| < \delta \quad \land \quad \inf_{t\in[0,T]}\|Z_A(t)-Z_D(t)\| \geq \epsilon,
    \]
    
    \item Effective Interception:
    \[
    \exists t \leq T: \|Z_A(t)-Z_D(t)\| < \epsilon,
    \]
\end{itemize}
where $\delta$ defines perimeter-defense proximity tolerance.

\begin{table}[!t]
\caption{Key Notation and Parameters}
\label{tab:notation}
\centering
\begin{tabular}{@{}lll@{}}
\toprule
Symbol & Definition & Domain/Constraints \\ 
\midrule
$\psi$ & Azimuth (yaw angle) & $[0, 2\pi)$ \\
$\phi$ & Polar (pitch angle) & $[0, \pi/2]$ \\
$r$& Radial distance from origin & $r \geq 0$ \\
$Z_A$ & Attacker state & $(\psi_A, \phi_A, r_A)$\\
$Z_D$ & Defender state & $(\psi_D, \phi_D, r_D)$\\
$Z_B$ & Breach point & $(\psi_B, \phi_B, R) \in \partial\mathcal{H}_R$ \\
$\|.\|$ &Euclidian distance& $\|.\| \geq 0$ \\
$ETA$ & Estimated Time of Arrival & $ETA \geq 0$\\
$\tau_A$ & Attacker ETA & $\|Z_A - Z_B\|/v$ \\
$\tau_D$ & Defender ETA & $\|Z_D - Z_B\|$ \\
$\epsilon$ & Capture radius & $\epsilon \in (0, 0.2R)$ \\
\bottomrule
\end{tabular}
\end{table}

\subsection{Formal Definition}
Follow \cite{ref10}, the temporal payoff function is defined as:
\begin{equation}
P(Z_D,Z_A,Z_B) = \tau_D - \tau_A = \|Z_D - Z_B\| - \frac{1}{v}\|Z_A - Z_B\|.
\end{equation}
Outcomes are determined by:
\[
\text{sgn}(P) = 
\begin{cases} 
1 & \text{(Attacker wins if } P > 0) \\
-1 & \text{(Defender wins if } P < 0).
\end{cases}
\]
The defender minimizes $P$ through anticipatory interception strategies, while the attacker maximizes $P$ via optimal breach point selection under the constraint $Z_B\in\partial\mathcal{H}_R$.

\begin{theorem}[Nash Equilibrium Strategy]
\label{thm:nash}
The following pair of strategies forms a Nash equilibrium in perimeter-defense game:
\begin{itemize}
    \item \textbf{Attacker's Strategy}: The attacker selects the optimal breach point \(B^*\) and follows the direct linear trajectory \(\overline{AB^*}\) at maximum speed \(v\), where \(B^*\) maximizes the temporal advantage:  
  \[
  B^* = \argmax_{B \in \partial\mathcal{H}_R} P(B) = \|DB\| - \tfrac{1}{v}\|AB\|,
  \]  

 \item \textbf{Defender's Strategy}: The defender intercepts along the direct linear trajectory \(\overline{DB^*}\) at maximum speed 1.

This strategy pair ensures neither player can unilaterally deviate to improve their payoff.
\end{itemize}
\end{theorem}

\begin{proof}
\textbf{Part 1: Attacker Optimality}. Fix $\sigma_D^*$. For any $B \neq B^*$, by $B^*$'s definition we have:
\begin{equation}
    P(B^*) \geq P(B).
\end{equation}
The attacker's linear trajectory $\overline{AB^*}$ at maximum speed $v$ minimizes the travel time $\tau_A$, as any alternative path would result in a longer distance and thus a longer travel time.

\textbf{Part 2: Defender's Optimal Strategy}. Fix the attacker's strategy of targeting $B^*$. To analyze potential deviations, we construct an equivalent geometric representation by scaling the attacker's initial position $A$ to a new point $A'$ such that
\begin{equation}
\|A'B^*\| = \frac{1}{v}\|AB^*\|,
\end{equation}
thereby normalizing the attacker’s speed to 1, which matches that of the defender. 
For any defender deviation to a point \( C \in \overline{AB^*} \), the payoff difference satisfies:
\begin{align}
\|DC\| - \|A'C\| &< \|DB^*\| - \|A'B^*\|, \label{eq:payoff_diff} \\
\intertext{which can be rewritten as:}
\|DC\| - \|A'C\| &< \|DB^*\| - \|A'C\| - \|B^*C\|. \label{eq:payoff_diff_simplified}
\end{align}
It implies that any deviation to \( C \) would require:
\begin{align}
\|DB^*\| > \|B^*C\| + \|DC\|. \label{eq:triangle_inequality}
\end{align}
Thus, no profitable deviation exists since (\ref{eq:triangle_inequality}) violates the triangle inequality. \qedhere
\end{proof}

\begin{figure}[!t]
\centering
\begin{subfigure}[b]{1.7in}
    \includegraphics[width=\linewidth]{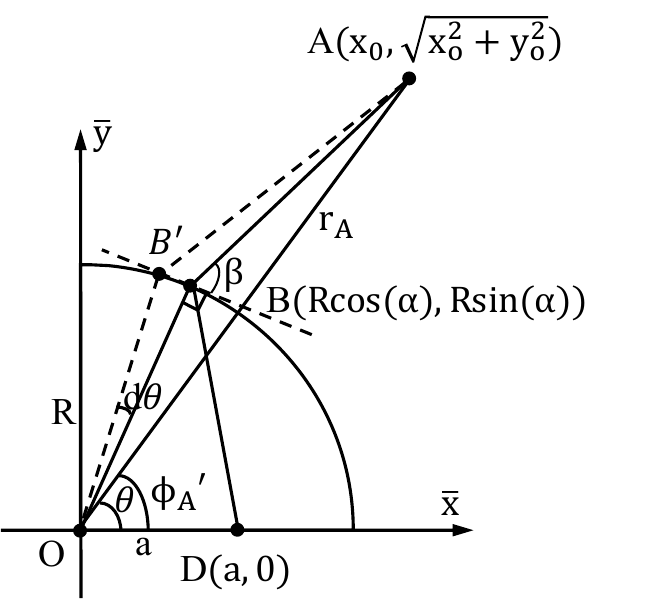}
    \caption{}
    \label{fig:2d}
\end{subfigure}
\hfil
\begin{subfigure}[b]{1.7in}
    \includegraphics[width=\linewidth]{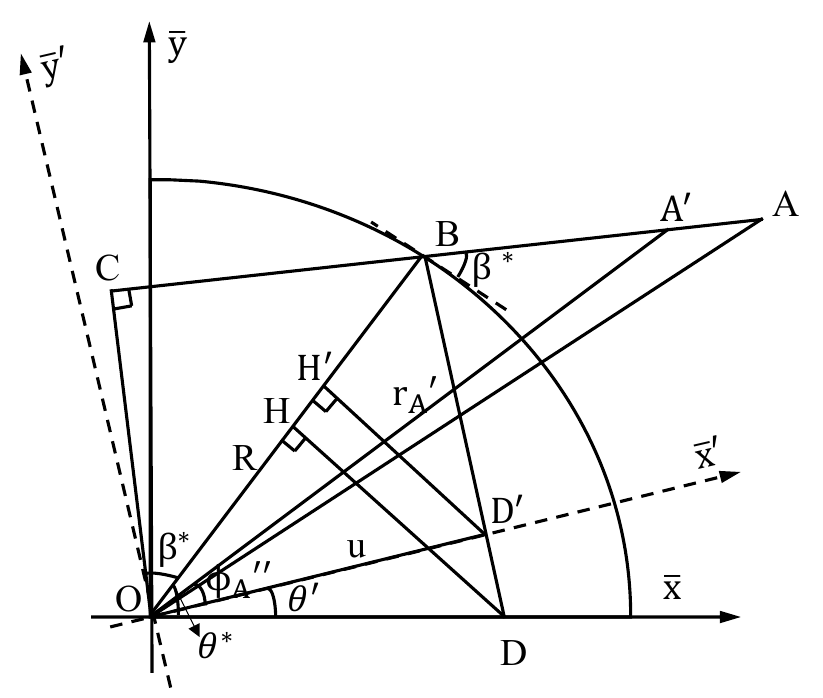}
    \caption{}
    \label{similar}
\end{subfigure}
\caption{(a) The Transformed 2D Plane. (b) Proportional distances maintained through similarity.}
\label{nash_verify}
\end{figure}

\subsection{Nash Equilibrium Solution}
By aligning the coordinate system such that defender position $D$ lies on the $x$-$z$ plane ($y=0$), we define the positions:
\[
B(x,y,z) \in \partial\mathcal{H}_R,\quad D(a,0,b),\quad A(x_0,y_0,z_0).
\]
The payoff function becomes:
\begin{multline}
P(x,y,z) = \underbrace{\sqrt{(x-a)^2 + y^2 + (z-b)^2}}_{d_D} \\
- \frac{1}{v}\underbrace{\sqrt{(x-x_0)^2 + (y-y_0)^2 + (z-z_0)^2}}_{d_A}.
\label{eq:payoff}
\end{multline}
Constrained optimization is formulated as:
\begin{equation}
\max_{x,y,z} P(x,y,z) \quad \text{s.t.} \quad x^2 + y^2 + z^2 = R^2.
\label{eq:opt}
\end{equation}
The Lagrangian function is:
\begin{equation}
\mathcal{L} = d_D - \frac{d_A}{v} - \lambda\left(x^2 + y^2 + z^2 - R^2\right).
\label{eq:Lagrangian}
\end{equation}
By computing the partial derivatives of $\mathcal{L}$ with respect to $x$, $y$, and $z$, and setting them to zero, the critical ratio is derived as:
\begin{equation}
\frac{\partial\mathcal{L}/\partial y}{\partial\mathcal{L}/\partial z} = \frac{y - \frac{d_D}{vd_A}(y-y_0)}{z - \frac{d_D}{vd_A}(z-z_0)} = \frac{y}{z}.
\label{eq:ratio}
\end{equation}
For $b=0$ (defender on $x$-axis), equation \eqref{eq:ratio} reduces to:
\begin{equation}
y_0 z = y z_0 \Rightarrow \pi: -z_0 y + y_0 z = 0.
\label{eq:plane}
\end{equation} 
This establishes that the optimal breach point $B^*$ must lie on plane $\pi$ containing the origin and orthogonal to vector $(0,-z_0,y_0)$ and the coplanarity of points $D$, $A$, and $O$ can be verified by the plane equation(\ref{eq:plane}). We construct an orthonormal coordinate system that projects the 3D defense game onto a 2D plane containing the critical points $O$, $D$, and $A$. The transformation matrix:
\begin{equation}
T = \begin{pmatrix}
1 & 0 & 0 \\
0 & \tfrac{y_0}{s} & \tfrac{z_0}{s} \\
0 & -\tfrac{z_0}{s} & \tfrac{y_0}{s}
\end{pmatrix}, \quad s = \sqrt{y_0^2 + z_0^2},
\end{equation}
aligns the system such that all strategic interactions occur in the $\bar{x}\bar{y}$-plane, where the transformed coordinates become:
\begin{align}
A &\mapsto (\bar{x}_A, \bar{y}_A, 0) = (x_0, s, 0) \\
D &\mapsto (\bar{x}_D, 0, 0) = (a, 0, 0) \\
B &\mapsto (R\cos\alpha, R\sin\alpha, 0).
\end{align}
This reduction enables 2D geometric analysis as shown in Fig. \ref{fig:2d}, where the optimal breach point $B^*$ lies on the transformed equatorial plane and \( \beta \) denotes the angle between the breach point and the tangent to the hemisphere.

The key geometric relationships derive from spherical trigonometry and optimization:

\begin{theorem}\label{thm:breach}
The optimal breach angle $\theta^*$ satisfies the fixed-point system:
\begin{equation}
\begin{cases}
\label{eq:brea_cond}
\theta^* = \phi_A' - \beta^* + \cos^{-1}\left(\dfrac{R\cos\beta^*}{r_A}\right) \\
\beta^* = \cos^{-1}\left(\dfrac{av\sin\theta^*}{\sqrt{a^2+R^2-2aR\cos\theta^*}}\right),
\end{cases}
\end{equation}
where $\phi_A' = \arctan\left(\dfrac{\sqrt{y_0^2+z_0^2}}{|x_0|}\right)$ defines the effective pitch angle.
\end{theorem}

\begin{proof}
1. \textbf{Key Distances}: Applying the law of cosines to both agents:
\begin{align}
\tau_D &= \sqrt{a^2 + R^2 - 2aR\cos\theta} \\
\tau_A &= \frac{\sqrt{r_A^2 + R^2 - 2r_AR\cos(\theta-\phi_A')}}{v}.
\end{align}

2. \textbf{Small-Angle Approximation}: For infinitesimal $d\theta$, trajectories $AB$ and $AB'$ become parallel with:
\begin{equation}
d\tau_A =  \frac{\|BB'\|\cos\beta}{v} = \frac{R\cos\beta}{v}d\theta.
\end{equation}

3. \textbf{Optimality Condition}: Setting $\frac{dP}{d\theta}=0$ yields:
\begin{equation}
\frac{aR\sin\theta}{\sqrt{a^2+R^2-2aR\cos\theta}} = \frac{R\cos\beta}{v}. 
\tag{Deriving $\beta^*$}
\end{equation}

4. \textbf{Geometric Constraint}: From triangle decomposition:
\begin{equation}
\label{eq:ab}
|AB| = \sqrt{r_A^2 - R^2\cos^2\beta} - R\sin\beta.
\end{equation}

5. \textbf{Angle Coupling}: Applying the cosine law to $\triangle ABO$ yields:
\begin{equation}
|AB|^2 = r_A^2 + R^2 - 2r_AR\cos(\theta-\phi_A').
\end{equation}
Substituting Eq. (\ref{eq:ab}) and simplifying leads to the angular coupling:
\begin{equation}
\cos(\theta-\phi_A') = \sin\left[\beta + \sin^{-1}\left(\frac{R\cos\beta}{r_A}\right)\right].
\end{equation}
This derives $\theta^*$ in Theorem \ref{thm:breach}, completing the proof.
\end{proof}

The 3D coordinates emerge from rotating the 2D solution $(R\cos\theta^*, R\sin\theta^*,0)$ back to original axes:
\begin{equation}
B^* = \left( R\cos\theta^*,\ \dfrac{y_0R\sin\theta^*}{\sqrt{y_0^2+z_0^2}},\ \dfrac{z_0R\sin\theta^*}{\sqrt{y_0^2+z_0^2}} \right)^T.
\end{equation}

\subsection{Dynamic Stability Analysis}
\begin{theorem}\label{thm:stability}
The optimal breach point $B^*$ remains invariant under simultaneous motion of attacker $A\to A'$ and defender $D\to D'$ at maximum speeds.
\end{theorem}

\begin{proof}
1. \textbf{Constrained Trajectories}: Under maximum-speed motion, ensuring collinearity with $B^*$:
\begin{equation}
A' \in \overline{AB^*}, \quad D' \in \overline{DB^*}, \quad \forall \Delta t > 0,
\end{equation}

2. \textbf{Verification of $(\theta^* - \theta', \beta^*)$}: We verify that $(\theta^* - \theta', \beta^*)$ satisfies the iterative update rules. First, rotate the coordinate system so that $D'$ lies on the $x$-axis, as shown in Fig.~\ref{similar}. The perpendicular from $O$ to $AB$ intersects at $C$, so $\angle BOC = \beta^*$. In $\triangle OA'C$, we derive:
\begin{equation}
\cos^{-1} \left(\frac{\|OC\|}{r_A'}\right) = \cos^{-1} \left(\frac{R \cos \beta^*}{r_A'}\right) = \angle A'OC.
\end{equation}
Thus, as illustrated in Fig.~\ref{similar}, the following holds:
\begin{multline}
\phi_A'' - \beta^{*} + \cos^{-1} \left(\frac{R \cos \beta^*}{r_A'}\right) = \phi_A'' - \beta^{*} + \angle A'OC \\
= \angle BOD' = \theta^* - \theta',
\end{multline}
confirming that $(\theta^* - \theta', \beta^*)$ satisfies the iterative update (\ref{eq:brea_cond}) for $\theta^*$. Next, from Fig.~\ref{similar}, the similarity of triangles ensures:
\begin{equation}
\frac{u \sin(\theta^* - \theta')}{\sqrt{u^2 + R^2 - 2uR \cos(\theta^* - \theta')}} 
= \frac{\|H'D'\|}{\|BD'\|} = \frac{\|HD\|}{\|BD\|}.
\end{equation}
This confirms that $(\theta^* - \theta', \beta^*)$ also satisfies the update rule (\ref{eq:brea_cond}) for $\beta^*$. Hence, as the defender and attacker move from $(D, A)$ to $(D', A')$, the optimal breach point $B^*$ remains invariant.
\end{proof}

\subsection{Winning Region Characterization}
\begin{theorem}[Zero-Payoff Surface]\label{thm:zerosurface}
For any given defender position \( Z_D \), there exists a unique, closed, and continuously differentiable surface \( \mathcal{S}(Z_D) \) such that  
\begin{equation}
P(Z_D, Z_A, Z_B) = 0, \quad \forall Z_A \in \mathcal{S}(Z_D).
\end{equation}
\end{theorem}

\begin{proof}
We first apply a coordinate transformation to align \( Z_D \) with the \( x \)-axis, yielding \( Z_D = (a,0,0) \). The attacker's position in the \( xy \)-plane is parameterized using polar coordinates \( (r_A, \theta) \). The optimal breach angle \( \beta^* \) is determined by Eq. (\ref{eq:brea_cond}). By enforcing the condition that the defender and attacker reach the interception point simultaneously, i.e., \( \tau_D = \tau_A \), we derive the set of feasible attacker positions \( Z_A \). Finally, revolving this solution around the \( x \)-axis generates the desired three-dimensional surface. \qedhere
\end{proof}
 
\section{System model and problem statement}
\label{sec:marl_for_pg}
In real-world scenarios, missiles exhibit considerable diversity in dynamics, control parameters, flight speeds, and are subject to complex environmental factors such as wind fields. To achieve a more realistic validation, we introduce a large-scale heterogeneous perimeter-defense game under the following assumptions. Attackers are assumed to possess unknown dynamics, allowing them to navigate wind perturbations and other environmental uncertainties with robustness. Their objective is to reach an optimal breach point at maximum velocity. In contrast, defenders are algorithmically controlled and tasked with intercepting the attackers while contending with wind perturbations and diverse dynamic constraints. Furthermore, to avoid flight interference and the risk of explosions from nearby allied missiles, defenders are required to maintain a safe separation distance.

Building on the one-on-one nash equilibrium strategies for both attackers and defenders, we utilize the Hungarian algorithm \cite{ref43} for optimal target assignment in the large-scale perimeter-defense game, ensuring efficient coordination and use MARL to solve the problem.

\subsection{Dec-POMDP}
To model the defender's decision-making in the perimeter-defense game, we formulate the problem as a decentralized partially observable Markov decision process (Dec-POMDP), represented by the tuple $M = (N, S, \{A_i\}_{i=1}^{n}, \{O_i\}_{i=1}^{n}, R, T, \{U_i\}_{i=1}^{n}, \gamma)$. Here, $N$ denotes the set of agents, while $S$, $\{A_i\}$, and $\{O_i\}$ represent the state space, individual action spaces, and individual observation spaces, respectively. The joint action and observation spaces are given by $A = A_1 \times A_2 \times \dots \times A_n$ and $O = O_1 \times O_2 \times \dots \times O_n$. At each time step $t$, the state of the environment is $s_t \in S$, and the agents take a joint action $a_t \in A$, which leads to the next state following the transition dynamics $s_{t+1} \sim T(s_{t+1} | s_t, a_t)$.

\subsection{State Transition and Action Space}
To enhance realism, we incorporate wind field perturbations and dynamic variations into the environment. The wind-induced perturbations are computed using a model that considers the agent's information, with the wind speed defined as:  
\begin{equation}  
\mathbf{w} = f(\text{height}, \mathbf{p}, \mathbf{v}) + \mathcal{N}(0, \sigma^2),  
\end{equation}  
where \(f(\cdot)\) captures systematic variations in wind speed based on the agent's height, position \(\mathbf{p}\), and velocity \(\mathbf{v}\). The term \(\mathcal{N}(0, \sigma^2)\) represents Gaussian noise with mean 0 and variance \(\sigma^2\), accounting for the stochastic nature of the wind. The agent's velocity is influenced by both the wind perturbations and its own acceleration  \(\mathbf{a}\) resulting from the agent's dynamics, with the velocity update formulated as: 
\begin{equation}  
\mathbf{v}^{t+1} = \mathbf{v}^t + \mathbf{w} + \mathbf{a} \Delta t,  
\end{equation}  

We consider three types of motion dynamics for agent control: (i) missile dynamics, where the agent manipulates angular velocities in the yaw and pitch directions to adjust its trajectory \cite{ref44}; (ii) acceleration-based dynamics, where the agent controls the increments of the yaw and pitch angles to influence the rate of directional change; and (iii) velocity-based dynamics, where the agent directly sets the yaw and pitch angles to determine the velocity direction. Thus, the action for agent \(i\) is defined by its dynamics type \(t_i\), allowing it to control either its heading angle, acceleration, or velocity vector adjustments by a two-dimensional continuous vector. Additionally, we incorporate variations in wind field parameters and velocity settings across these control paradigms, resulting in a total of 16 distinct dynamic types in the environment.

\subsection{Observation Space}
The observation space for agent \( i \) in the interception scenario is defined as follows:

\begin{align}
    \text{Observation}_i = 
    \Bigg\{ 
    &\underbrace{(x_i, y_i, z_i, \psi_i, \theta_i, v_{xi}, v_{yi}, v_{zi}, t_i)}_{\text{Self-Information}}, \nonumber \\
    &\underbrace{\{(d_{ij}, \theta_{ij}, \phi_{ij}) \mid j \neq i \}}_{\text{Other Agents' Information}}, \nonumber
    \\
    &\underbrace{(d_{iA}, \theta_{iA}, \phi_{iA})}_{\text{Target Information}} 
    \Bigg\}.
\end{align}
In this formulation, the \emph{self-information} component describes the intrinsic state of agent \( i \), comprising its position \( (x_i, y_i, z_i) \), orientation angles \( (\psi_i, \theta_i) \), velocity components \( (v_{xi}, v_{yi}, v_{zi}) \), and current dynamics type \( t_i \). Here, \( \psi_i \) and \( \theta_i \) represent the yaw and pitch angles, respectively. The \emph{other agents' information} encodes the relative state of nearby agents, where \( d_{ij} \) denotes the Euclidean distance to agent \( j \), and \( (\theta_{ij}, \phi_{ij}) \) are the relative pitch and yaw angles with respect to agent \( j \). The \emph{target information} captures the state of the assigned attacker from the perspective of agent \( i \), specified by the distance \( d_{iA} \) and the corresponding pitch and yaw angles \( (\theta_{iA}, \phi_{iA}) \).

\subsection{Reward}
The reward function is designed to guide each defender toward efficient and safe interception of its assigned attacker. For agent \( i \), the reward at time step \( t \) is defined as:
\begin{equation}
r_{it} = r_{it}^{\mathrm{task}} + r_{it}^{\mathrm{guide}} + r_{it}^{\mathrm{collide}},
\end{equation}
where \( r_{it}^{\mathrm{task}} = \alpha_1 + \alpha_2 \cdot \mathbb{I}(d_{iA}^{t} < d_{\mathrm{th}}) \) imposes a stepwise penalty (\( \alpha_1 = -0.01 \)) to encourage timely completion, and grants a sparse reward (\( \alpha_2 = 10 \)) when the agent successfully intercepts its assigned attacker, i.e., when the distance \( d_{iA}^{t} \) falls below a predefined threshold \( d_{\mathrm{th}} \) and the indicator function \( \mathbb{I}(\cdot) \) returns 1 if the condition holds, and 0 otherwise. To encourages defender to reduce its distance to its assigned attacker, the guidance term \( r_{it}^{\mathrm{guide}} = \alpha_3 \cdot (d_{iA}^{t-1} - d_{iA}^{t}) \) with \( \alpha_3 = 10 \) provides dense feedback based on the reduction in attacker distance across consecutive steps. Finally, potential collisions are penalized through \( r_{it}^{\mathrm{collide}} = \alpha_4 \cdot \sum_{j \ne i} \mathbb{I}(d_{ij}^{t} < d_{\mathrm{safe}}) \), where \( \alpha_4 = -0.03 \) and \( d_{\mathrm{safe}} \) denotes a minimum safety distance.

\section{Embedded Mean Field Reinforcement Learning}
\label{sec:emfac_method}
To tackle the complex large-scale tasks, we propose \textit{Embedded Mean Field Actor-Critic} (EMFAC), a novel reinforcement learning framework, the overall architecture is illustrated in Fig. \ref{emfac_method}. EMFAC consists of two key components: the High-level Action Learning Module, which leverages state representations to learn high-level mean-field actions, effectively decoupling the impact of heterogeneous motion dynamics and accelerating the training process; and the Agent-level Attention Module, which derives agent-specific lightweight attention based on reward representations, enabling selective extraction of the most relevant agent information from both the mean-field action space and the state space. In the following sections, we conduct an in-depth analysis of each module.

\begin{figure*}[htp]
    \centering
    \includegraphics[width=2\columnwidth]{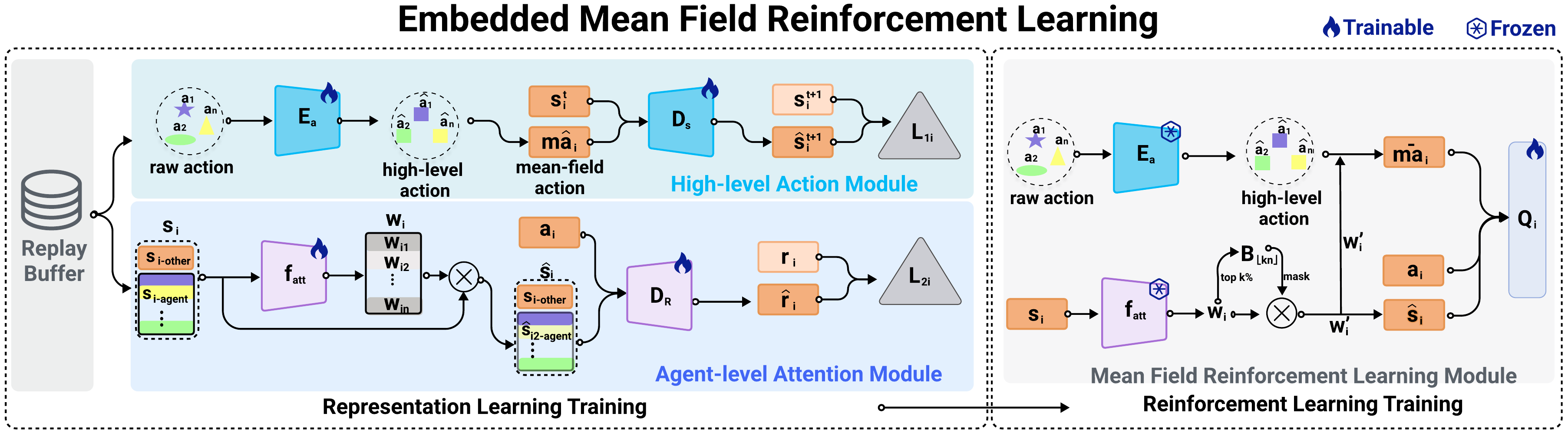} 
    \caption{Framework of the EMFAC Method. The left part shows the training of the high-level action encoder \(E_a\) and attention module \(f_{\text{att}}\), while the right illustrates the mean-field RL process using them. Both modules are trained alternately.}
    \label{emfac_method}
\end{figure*}

\subsection{Embedded High-level Mean Field Action}
To address the challenge of heterogeneous motion dynamics and enhance the expressiveness of mean-field representations, we introduce a \textit{High-level Action Encoder}, which encodes raw heterogeneous actions into a unified high-level action space. This transformation enables more effective mean-field aggregation. The core idea behind mean-field aggregation is to model the influence of other agents on the central agent. To train the high-level action encoder $E_a$, we introduce an auxiliary task from the perspective of state prediction. Specifically, we apply $E_a$ to encode a raw action $a$ into an abstract high-level action $\hat{a}$. The high-level action is first used to compute the mean-field action $\hat{ma}_i$, which is then concatenated with the current state $s^t_i$ and passed through a \textit{State Decoder} $D_s$ to predict the next-step state $\hat{s}^{t+1}_i$. To enhance robustness, we employ the Huber loss \cite{ref45} for training, which mitigates the impact of outliers and stabilizes learning. The model is trained by minimizing the loss between the predicted and true next-step state $s^{t+1}_i$. The formal equations are given as follows:
\begin{equation}
    E_a(a_1, a_2, \dots, a_n) = (\hat{a}_1, \hat{a}_2, \dots, \hat{a}_n),
\end{equation}
\begin{equation}
\hat{ma}_i = \frac{1}{N_i} \sum_{j \in N(i)} \hat{a}_i,
\end{equation}
\begin{equation}
    D_s(\hat{ma}_i, s_i^t) = \hat{s}_i^{t+1}, \quad i = 1, \dots, n,
\end{equation}
\begin{equation}
\label{eq:high_action_loss}
    \mathcal{L}_1 = \frac{1}{n} \sum_{i=1}^{n} \text{HuberLoss}(\hat{s}_i^{t+1}, s_i^{t+1}).
\end{equation}
Here, \( N(i) \) denotes the set of neighboring agents for agent \( i \), with \( N_i \) being the number of such neighbors. By leveraging high-level mean-field action representations, agents can more effectively capture the influence of others, improving their understanding of observation and reward structures.

\subsection{Agent-level Attention Module}
In large-scale multi-agent tasks, the vast number of agents, high-dimensional state spaces, and frequent agent interactions introduce significant challenges for reinforcement learning. The reward function often entangles the influences of multiple agents, making exploration inefficient and hindering the ability to focus on critical information. To address these challenges, we propose an \textit{Agent-level Attention Module} based on reward representation learning, which enables agents to selectively attend to other agents that are most relevant to task completion. This attention mechanism is applied to both the state space and the high-level mean-field action, thereby enhancing the agent's ability to capture essential information.

To train the agent-level attention module \( f_{\text{att}} \), the agent's state \( s_i \), which includes information about other agents, is processed by \( f_{\text{att}} \), implemented as a multi-layer perceptron (MLP). It outputs agent-specific attention weights used to compute a weighted sum over the states, resulting in the attention-enhanced state \( \hat{s}_i \). This enhanced state, together with the agent's action \( a_i \), is passed to a parameterized \textit{Reward Decoder} \( D_R \) to predict the estimated reward \( \hat{r}_i \). The module is trained by minimizing the mean squared error (MSE) between the predicted and true rewards. The formal definitions are as follows:
\begin{equation}
    f_{\text{att}}(s_1, s_2, \dots, s_n) = (w_1, w_2, \dots, w_n),
\end{equation}
where $w_i = (w_{i1}, w_{i2}, \dots, w_{in})$ represents the agent-specific attention weights for agent $i$, with $w_{ij}$ denoting the attention weight of agent $i$ towards agent $j$. Using these weights, the attention-enhanced state is computed as:
\begin{equation}
    \hat{s}_i = (s_{i-other}, s_{i-agent} \cdot w_i),
\end{equation}
where $s_{i-other}$ denotes the agent-independent information of the state, while $s_{i-agent}$ represents information related to other agents. The reward prediction and training objective are defined as:
\begin{align}
    \hat{r}_i &= D_R(\hat{s}_i, a_i), \\
    \mathcal{L}_2 &= \frac{1}{n} \sum_{i=1}^{n} \text{MSELoss}(r_i, \hat{r}_i). \label{eq:att_loss}
\end{align}
By leveraging reward-guided attention, agents learn to filter out irrelevant interactions and focus on those critical to task success.

\subsection{Training of EMFAC}
During Q-network training, the parameters of $E_a$ and $f_{\text{att}}$ remain fixed. Each agent encodes raw actions into high-dimensional representations through $E_a$, while computing attention weights $w_i \in \mathbb{R}^n$ via $f_{\text{att}}$ based on states.  To encourage more focused interactions, we impose a constraint on the number of attended agents by applying a refined attention mechanism, which is formulated as: 
\begin{align}
B_{\lfloor k n \rfloor}(w_i) = 
\begin{cases} 
1 & \text{if } w_{ij} \in \text{top-}\lfloor k n \rfloor(w_i) \\
0 & \text{otherwise}
\end{cases}, \\
w'_i = \text{softmax}\left( w_i \odot B_{\lfloor k n \rfloor}(w_i) \right),
\end{align}
where $k \in (0,1)$ is the sparsity ratio, $\odot$ denotes element-wise multiplication, and $\lfloor \cdot \rfloor$ ensures integer truncation. Using these refined attention weights, we compute the attention-weighted high-level mean-field action $\bar{ma}_i$ and the attention-weighted state $\hat{s}_i$:
\begin{align}
\label{eq:ma_i}
\bar{ma}_i = \frac{1}{N_i} \sum_{j \in N(i)} \hat{a}_j \cdot w_{ij}^{'},\\
\label{eq:att_o_i}
\hat{s}_i = (s_{i-other}, s_{i-agent} \cdot w_i'),
\end{align}
It is worth noting that under the decentralized setting, the observation \( o_i \) can be processed in the same manner as the state \( s_i \) previously described, yielding a attention-weighted observation \( \hat{o}_i \) for training the Q-network. This design enables our method to be seamlessly applied to both CTDE (Centralized Training with Decentralized Execution) and DTDE (Decentralized Training and Decentralized Execution) paradigms, demonstrating its broad applicability and flexibility. The computed high-level mean-field action and the processed state are then fed into the agent’s Q-network for MFRL. Under the mean-field assumption, the $Q$-function can be approximated as:
\begin{equation}
    Q_{\phi_i}(\hat{s}_i, \hat{a}_1, \hat{a}_2,..., \hat{a}_n) \approx Q_{\phi_i}(\hat{s}_i, \hat{a}_i, \bar{ma}_i).
\end{equation}
We denote the target Q-network parameters by $\phi_i^-$, and $\hat{s}_i’$ represents the next attention-weighted state. The value function under the attention-weighted next state $\hat{s}_i'$ is computed as:
\begin{align}
v_i\left(\hat{s}_i'\right)=\sum_{\hat{a}_i} \pi_{\theta_i}\left(\hat{a}_i \mid \hat{s}_i'\right) \mathbb{E}_{\bar{ma}_i\left(\boldsymbol{\hat{a}}^{-i}\right) \sim \bar{\pi}^{-i}}\left[Q_{\phi_i^-}\left(\hat{s}_i', \hat{a}_i, \bar{ma}_i\right)\right].
\end{align}
The Q-network is trained using the following loss function:
\begin{align} 
\label{eq:y_i}
    y_i = r_i + \gamma   v_i\left(\hat{s}_i'\right),\\
\label{eq:critic_loss}
    \mathcal{L}(\phi_i) = \left(y_i - Q_{\phi_i}(\hat{s}_i, \hat{a}_i, \bar{ma}_i) \right)^2,
\end{align}
 To maintain timeliness during the inference phase, for the actor, instead of using attention-weighted observation $\hat{o}_i$, we utilize the original observations $o_i$ and the actor is updated using the policy gradient:
\begin{equation}
\label{eq:actor_loss}
    \nabla_{\theta_i} J(\theta_i) \approx \nabla_{\theta_i} \log \pi_{\theta_i} (o_i) Q_{\phi_i}(\hat{s}_i, \hat{a}_i, \bar{ma}_i) \bigg|_{a_i = \pi_{\theta_i}(o_i)}.
\end{equation}

By leveraging high-level action encoding and agent-level attention, EMFAC enhances the efficiency and scalability. In practical training, we concurrently conduct MFRL and the training of $E_a$ and $f_{\text{att}}$. During off-policy training, the training of $E_a$ and $f_{\text{att}}$ converges more rapidly. After a certain period of training and convergence, their parameters can be frozen, allowing us to focus solely on the training of reinforcement learning. This approach introduces minimal additional computational overhead to the algorithm's execution. The specific algorithmic procedure is illustrated in the pseudocode \ref{alg}.

\begin{algorithm}
\caption{EMFAC for Multi-agent Reinforcement Learning}
\label{alg}
\begin{algorithmic}[1]
\STATE Initialize $Q_{\phi_i}, Q_{\phi^-_i}, \pi_{\theta_i}, \pi_{\theta^-_i}$ for all $i \in \{1, \ldots, N\}$
\WHILE{training not finished}
    \FOR{each agent $i$}
        \STATE Sample action $a_i = \pi_{\theta_i}(o_i)$
        \STATE Take the joint action $a = [a_1, \ldots, a_N]$ and observe the reward $r = [r_1, \ldots, r_N]$ and the next state $s'$
        \STATE Store $\langle s, a, r, s' \rangle$ in replay buffer $\mathcal{D}$
    \ENDFOR
    \FOR{$i = 1$ to $N$}
        \STATE Sample a minibatch of $K$ experiences $\langle s, a, r, s' \rangle$ from $\mathcal{D}$
        \STATE Calculate $\bar{ma}_i$ and $\hat{o}_i$ by Eq. (\ref{eq:ma_i}) and Eq. (\ref{eq:att_o_i})
        \STATE Calculate $y_i$ by Eq. (\ref{eq:y_i})
        \STATE Update the critic by Eq. (\ref{eq:critic_loss})
        \STATE Update the actor by Eq. (\ref{eq:actor_loss})
    \ENDFOR
    \STATE Update the parameters of the target networks for each agent $i$ with learning rates $\tau_\phi$ and $\tau_\theta$:
    \[
    \phi^-_i \leftarrow \tau_\phi \phi_i + (1 - \tau_\phi) \phi^-_i
    \]
    \[
    \theta^-_i \leftarrow \tau_\theta \theta_i + (1 - \tau_\theta) \theta^-_i
    \]
    \STATE Sample a minibatch of $K$ experiences $\langle s, a, r, s' \rangle$ from $\mathcal{D}$ to train $E_a$ and $f_{\text{att}}$:
    \STATE Update $E_a$ by Eq. (\ref{eq:high_action_loss})
    \STATE Update $f_{att}$ by Eq. (\ref{eq:att_loss})
\ENDWHILE
\end{algorithmic}
\end{algorithm}

\section{Experiment}
\label{sec:experiment}
\subsection{Nash Equilibrium Verification}
\label{sec:nash_exp}
\subsubsection{Experimental Setup}
In this part, our validate the correctness of the proposed Nash equilibrium strategy, we conducted a series of simulation experiments under three strategic settings: (1) both defenders and attackers follow the optimal strategy, (2) only attackers adhere to the optimal strategy while defenders determine their interception points by computing the intersection of the attacker-defender line with the engagement hemisphere, and (3) only defenders follow the optimal strategy while attackers slightly deviate from the optimal breach point. Both sides were randomly initialized, with attackers positioned outside the perimeter and defenders inside, ensuring that under equilibrium conditions, both reached the optimal breach point simultaneously. The attacker’s speed was set to 1.0, while the defender’s speed was 0.8. 

\subsubsection{Simulation Results}
As illustrated in Fig. \ref{nash_verify_exp}, when both attacker and defender adhere to the Nash equilibrium strategy, they reach the optimal breach point simultaneously, resulting in a payoff of zero. When defender unilaterally deviates from the Nash strategy, its interception time increases, leading to a reduced payoff and a decline in defender performance. Conversely, when attacker unilaterally deviate, its breach time increases, resulting in a higher payoff but a decreased offensive effectiveness. These results empirically validate that the derived strategy constitutes a Nash equilibrium for both defenders and attackers. 

To better understand the strategic outcomes of the game, we visualize the decision perimeter that delineates the winning regions for both sides. Fixing the initial position of the defender, we identify the set of initial attacker positions $G$, where the time required for both agents to reach the optimal breach point is equal. As shown in Fig. \ref{nash_analy_exp}, the set $G$ forms a closed spherical surface, which aligns with our theoretical analysis. This perimeter serves as the critical threshold: if the attacker’s initial position lies outside the surface, the defender secures victory; conversely, if the initial position is within the surface, the attacker successfully breaks through.

\begin{figure*}[!t]
\centering
\subfloat[]{\includegraphics[width=1.4\columnwidth,height=2.8cm]{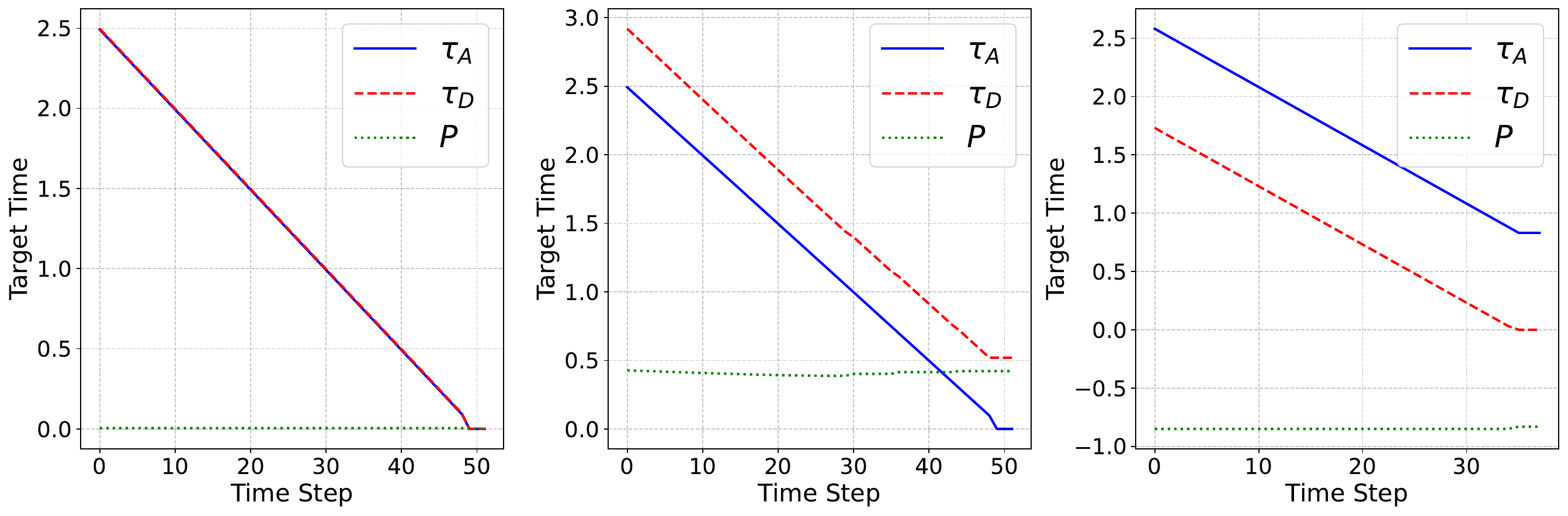}%
\label{nash_verify_exp}}
\hfil
\subfloat[]{\includegraphics[width=0.55\columnwidth,height=3.8cm]{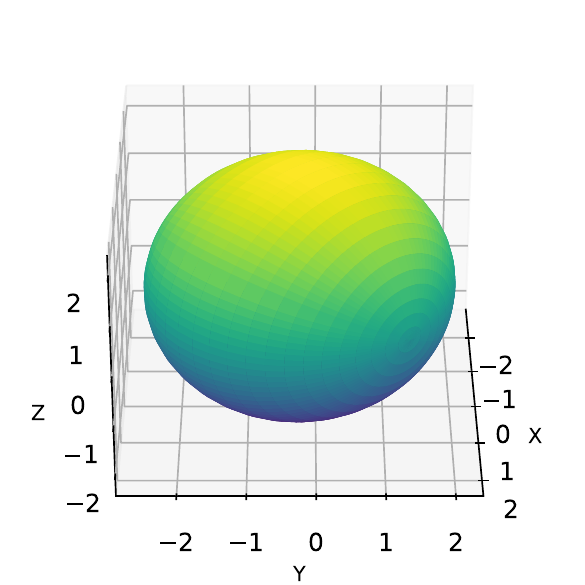}%
\label{nash_analy_exp}}
\caption{Simulation results for Nash equilibrium validation. (a) The time required for both attackers and defenders to reach the breach point and their corresponding payoffs. The three scenarios, from left to right, are: (i) both the attacker and defender adhering to the Nash equilibrium strategy, (ii) only the attacker following the Nash equilibrium strategy, and (iii) only the defender adopting the Nash equilibrium strategy. (b) Visualization of the decision perimeter that delineates the winning regions for both sides, with the initial defender position set at $[0.2, 0.2, 0.0]$.}
\label{nash_exp}
\end{figure*}

\subsection{Large-scale Heterogeneous Perimeter-defense Game Experiment}
\label{sec:peri_game_exp}
\subsubsection{Experimental Setup}
To evaluate the effectiveness of our algorithm, we run EMFAC on our large-scale heterogeneous perimeter-defense game. All algorithms are implemented within the same framework HARL\cite{ref25, ref26}, with identical architectures and hyperparameters. Follow \cite{ref27, ref31}, we compare EMFAC with several baseline methods.

\begin{itemize}
    \item \textbf{Rule-based strategy:} Defenders fly directly towards the optimal breach point at maximum speed, without accounting for wind perturbations. To prevent collisions, they randomly select a direction for interception when too close to other agents.
    \item \textbf{Independent learning methods:} We include IDDPG \cite{ref46} and ITD3 \cite{ref47} as baselines.
    \item \textbf{Multi-agent learning methods:} To assess the impact of cooperation, we compare our approach against MADDPG \cite{ref48} and MATD3.
    \item \textbf{Heterogeneity-aware algorithms:} To specifically evaluate methods designed for heterogeneous environments, we include HADDPG \cite{ref26} and HATD3 \cite{ref26}.
    \item \textbf{Scalability-oriented approaches:} Given the large-scale heterogeneous nature of our task, we also compare against the MTMFAC \cite{ref31} algorithm.
\end{itemize}

\subsubsection{Main Results}
We conducted experiments on perimeter-defense game with scales of 10, 20, 30, and 50 with 5 seeds. The task requires a comprehensive consideration of the interception success rate and the collision rate between defenders. Therefore, we adopted the average test reward as the most important evaluation metric. This metric takes into account the weights of the two indicators according to the designer's requirements and is consistent with the specific task requirements. The testing average reward curves are shown in the Fig. \ref{main_results}. On tasks of different scales, the EMFAC method achieved the best performance. With the help of representation learning and agent-level attention, EMFAC can more effectively learn the influence of other agents, attribute rewards, and focus on important information. This results in faster learning rates and better convergence performance. In particular, EMFAC shows a significant performance gap compared to other algorithms even in large scales. As the scale continues to expand, the observation space grows exponentially, leading to a certain degree of performance decline in all algorithms. However, EMFAC still achieved the best results. This demonstrates the robustness and effectiveness of EMFAC in handling complex and large-scale multi-agent tasks.

We have also demonstrated the average success and collision rates of various algorithms at convergence across different scales in Fig. \ref{main_rate}. It is observable that as the scale of perimeter-defense game expands, the observation and action spaces of the agents experience an exponential increase. Consequently, collisions become more frequent, and the interference of collision penalties in the reward further intensifies. As a result, the success rates of different algorithms decline, while collision rates rise. However, the EMFAC algorithm achieves the highest success rates and the lowest collision rates, showing a significant margin over other algorithms. This also verifies the effectiveness of the EMFAC method in large-scale tasks.

\begin{figure*}[htp]
    \centering
    \includegraphics[width=2\columnwidth]{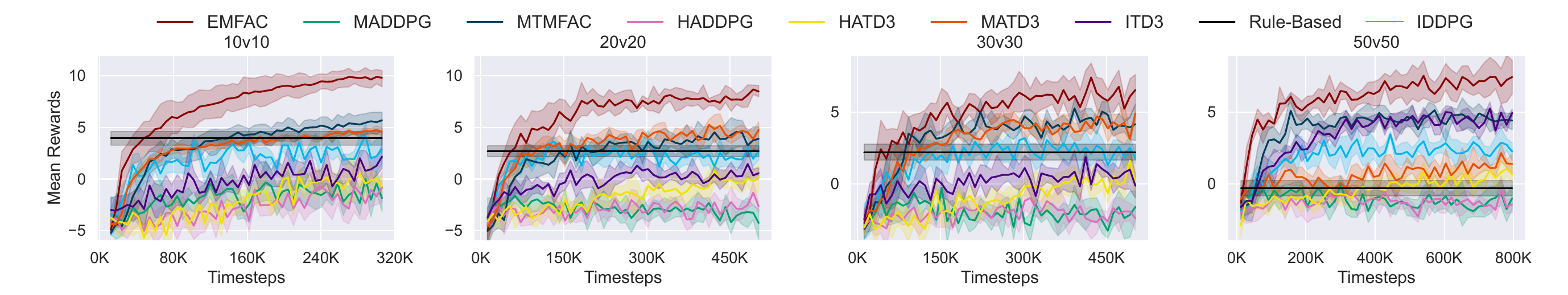} 
    \caption{Comparison of average reward learning curves for different algorithms in perimeter-defense game tasks of varying scales. (a)Success Rate. (b) Collision Rate.}
    \label{main_results}
\end{figure*}

\begin{figure*}[!t]
\centering
\subfloat[]{\includegraphics[width=1.0\columnwidth]{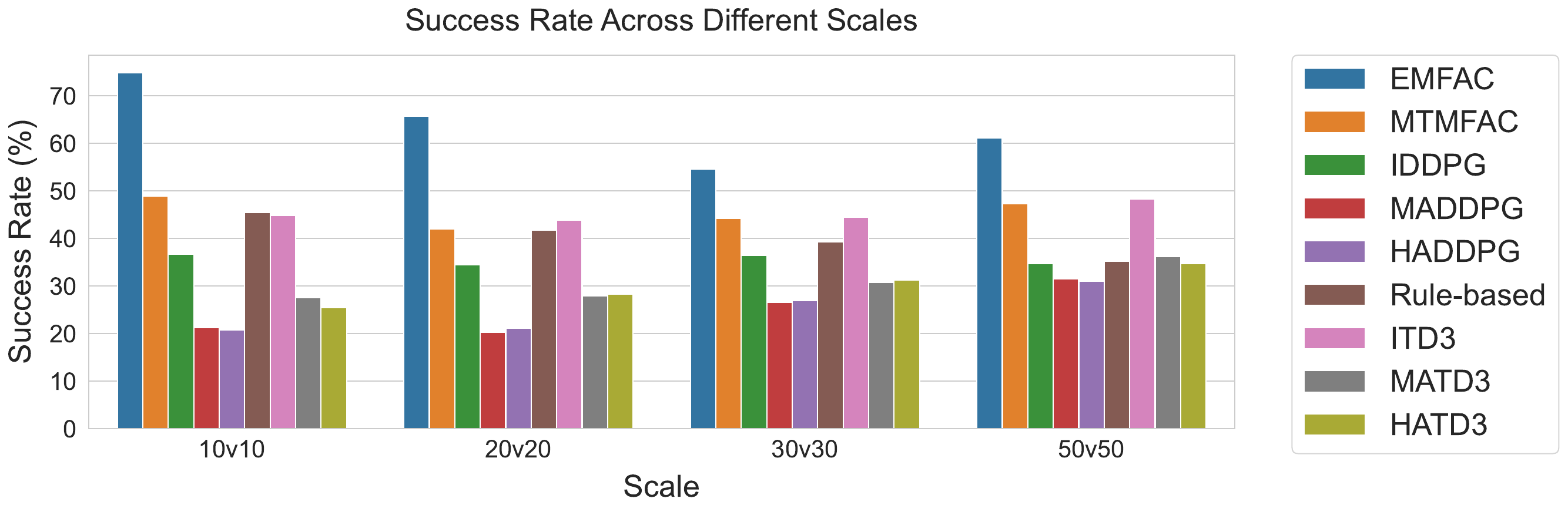}%
\label{main_success}}
\hfil
\subfloat[]{\includegraphics[width=1.0\columnwidth]{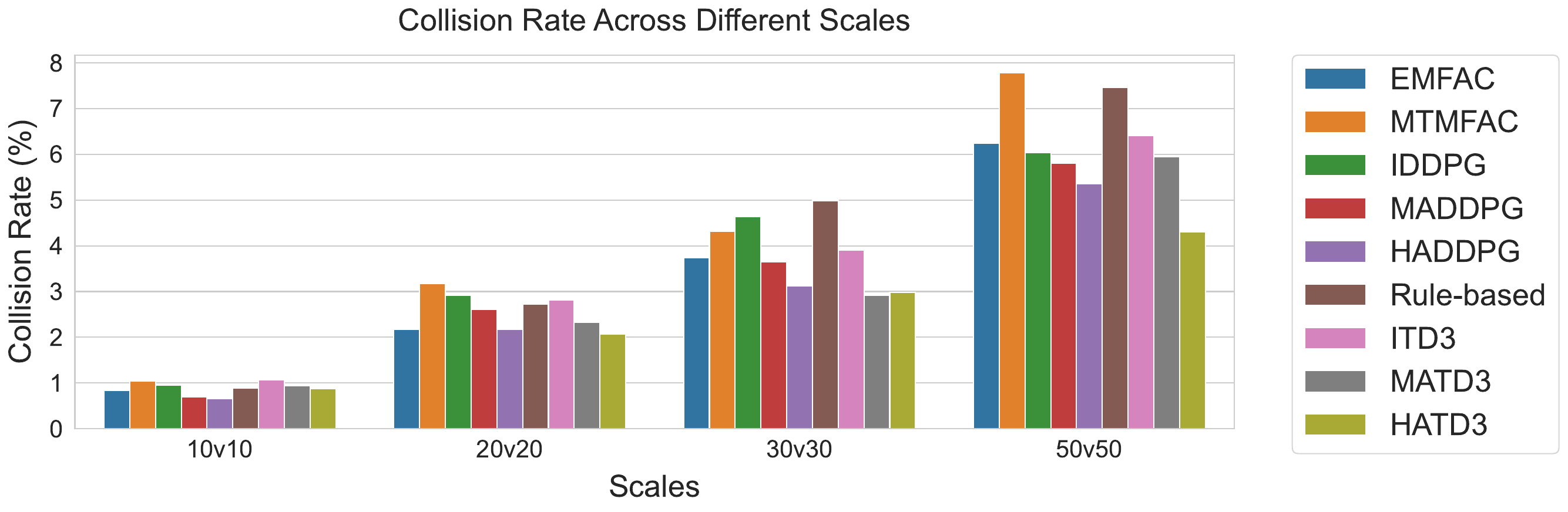}%
\label{main_coll}}
\caption{Comparison of Success and Collision Rates of Various Algorithms at Different Scales.}
\label{main_rate}
\end{figure*}

\subsubsection{High-level Action Visualization Analysis}
To visually demonstrate the high-level actions learned through representation learning, we conducted training on a 5v5 task involving two distinct types of motion dynamics, setting the high-level action dimension to 2. During training, we allowed converged agents to interact with the environment five times and visualized both the raw actions and the high-level actions, as shown in the Fig. \ref{high_action_plot}. In the raw actions, the distribution of actions across the two dynamics exhibits significant differences and a wide dispersion, with the variance for the two dimensions being 0.46 and 3.59, respectively. Such a broad and dispersed action space makes it difficult to effectively capture the influence of other agents using mean-field theory. In contrast, after abstracting the actions into high-level representations, the range of actions for both dynamics becomes much more compact and concentrated, with significantly smaller variances of 0.11 and 0.07. The more concentrated action space facilitates better learning by the Q-network. Furthermore, through representation learning, the heterogeneity of dynamics is effectively reduced, and the influence of other agents on the central agent is captured through higher-level mean-field action. This not only accelerates training but also demonstrates the rationality of our approach.
\begin{figure}[htp]
\centering
\includegraphics[width=6cm]{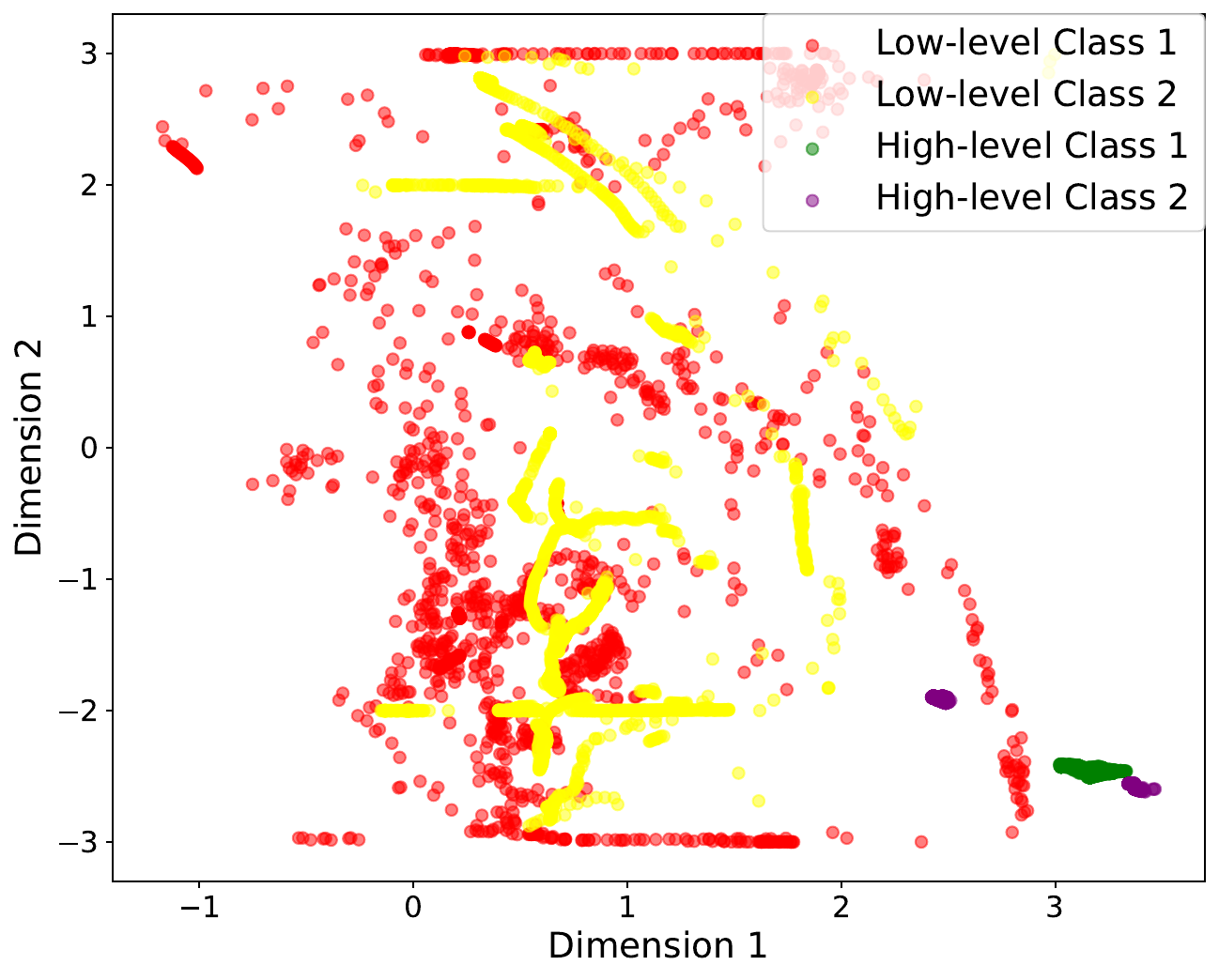}
\caption{High-level Action Visualizatoion}
\label{high_action_plot}
\end{figure}

\subsubsection{Analysis of Attention Mechanism Selection}
To provide a clearer and more intuitive demonstration of the rationale behind the attention mechanism’s selection of relevant agents, we conduct experiments in a 5v5 setting with an attention ratio $k=0.6$. Upon convergence, we extract an attention matrix for analysis, as shown in the Fig. \ref{att_analysis}. The influence of neighboring agents is primarily reflected in the obstacle avoidance task, meaning that the selected attention weights should align with this subtask. The visualization reveals that attention is primarily directed toward closer agents, as confirmed by the distance matrix, which highlights their proximity compared to others. Furthermore, the cosine similarity between the optimal breach direction and the bearing angles of surrounding agents serves as another crucial factor. A higher similarity suggests that these agents share a similar orientation and are more likely to be involved in future collisions, warranting greater attention. This observation is consistent with the learned attention patterns, confirming that agent-level attention should jointly consider both distance and relative orientation to enhance situational awareness and improve collision avoidance. These findings further validate the reasonability of our proposed method. 
\begin{figure}[htp]
\centering
\includegraphics[width=9.2cm]{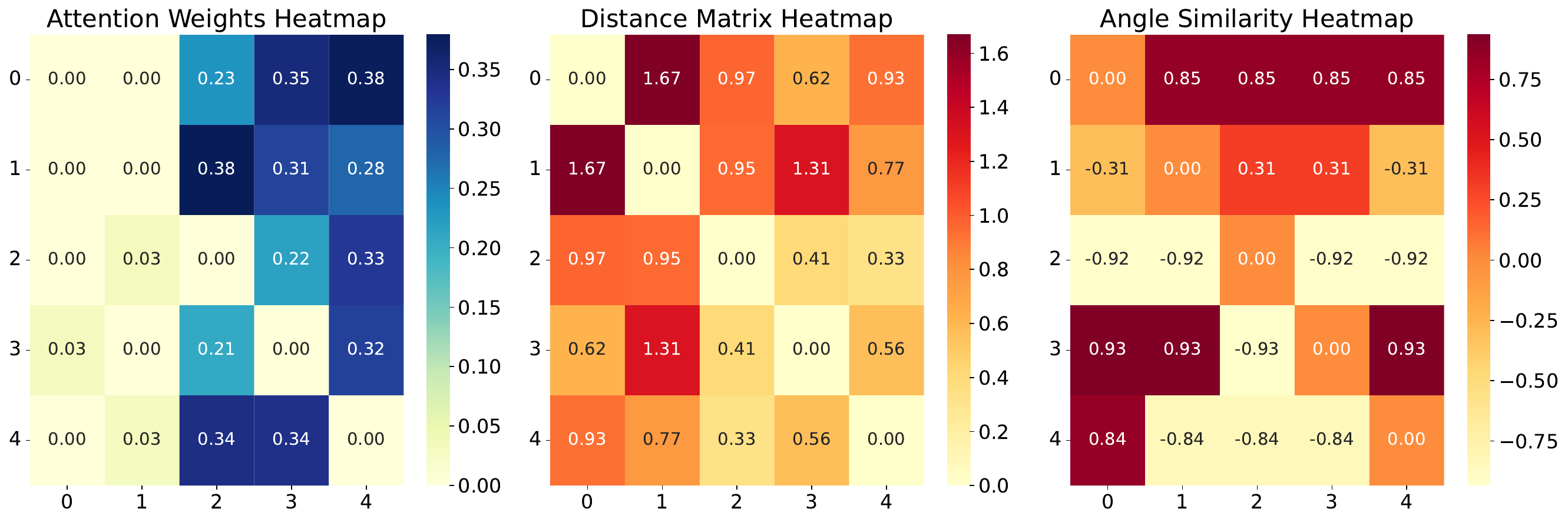}
\caption{Visualization of different matrices. From left to right: Attention matrix, Distance matrix, and Azimuthal Pre-similarity matrix.}
\label{att_analysis}
\end{figure}

\subsubsection{Flight Trajectory Visualization}
To effectively illustrate and analyze the learned policies of the agents, we conduct experiments in smaller-scale scenarios with 5v5 and 10v10 settings for clearer visualization, as shown in the Fig. \ref{fig:trajectory_comparison}. Due to the need to overcome complex dynamics, wind perturbations, and obstacle avoidance, the trajectories of the defender agents exhibit significant curvature. Nevertheless, they successfully reach their designated targets while avoiding collision.

\begin{figure}[!t]
\centering
\subfloat[]{\includegraphics[width=1.6in]{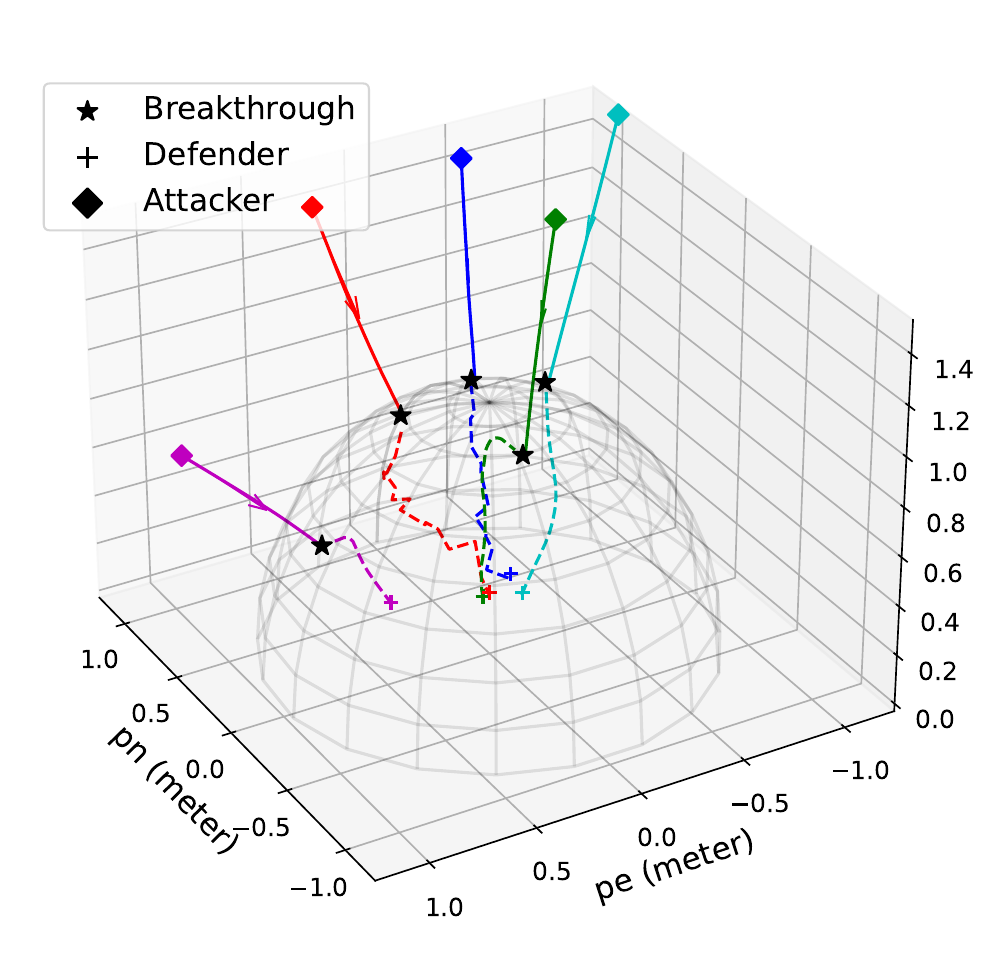}}
\hfil
\subfloat[]{\includegraphics[width=1.6in]{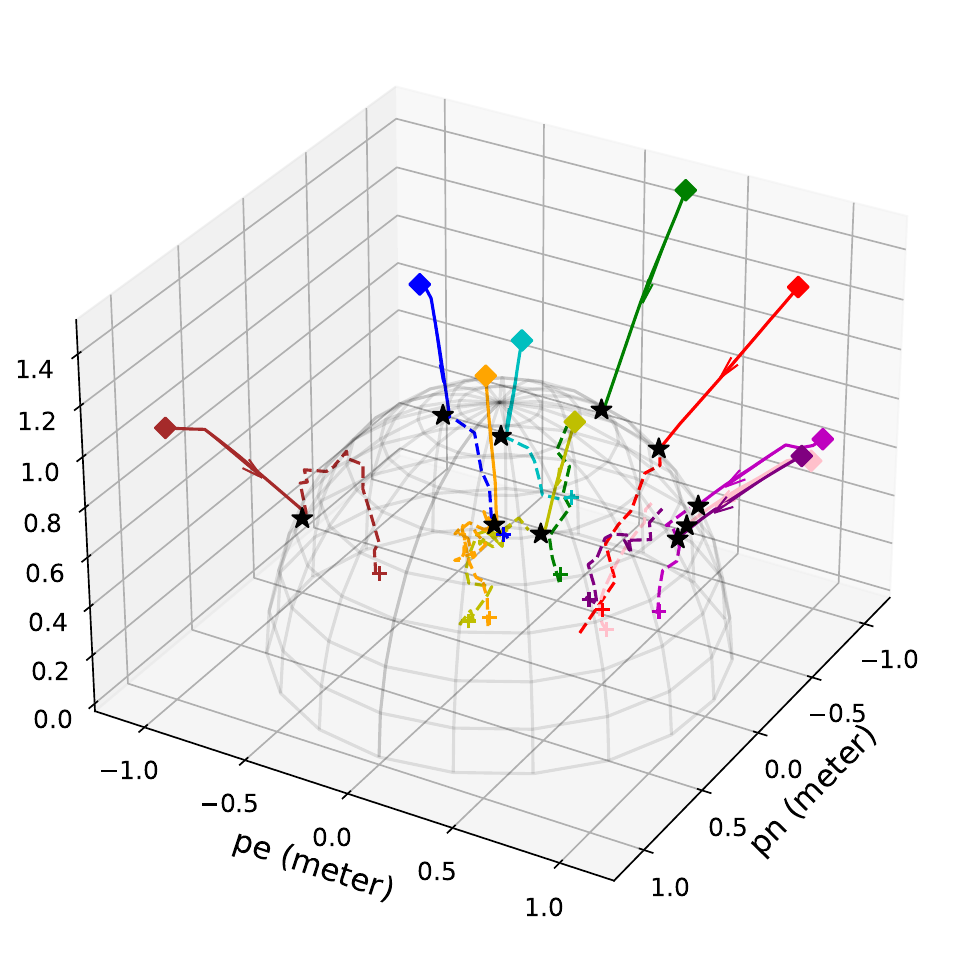}}
\caption{Comparison of 5v5 and 10v10 trajectory visualizations. (a) 5v5. (b) 10v10.}
\label{fig:trajectory_comparison}
\end{figure}

\subsubsection{Parameter Sensitivity Analysis}
To evaluate the impact of two key parameters in the EMFAC algorithm, we conducted a parameter sensitivity analysis using a 20v20 scenario. Specifically, we examined the effects of attention ratio $k$ and high-level action dimension on performance, the results are showed in Figure \ref{fig:explore}.

\begin{figure}[!t]
\centering
\subfloat[]{\includegraphics[width=1.6in]{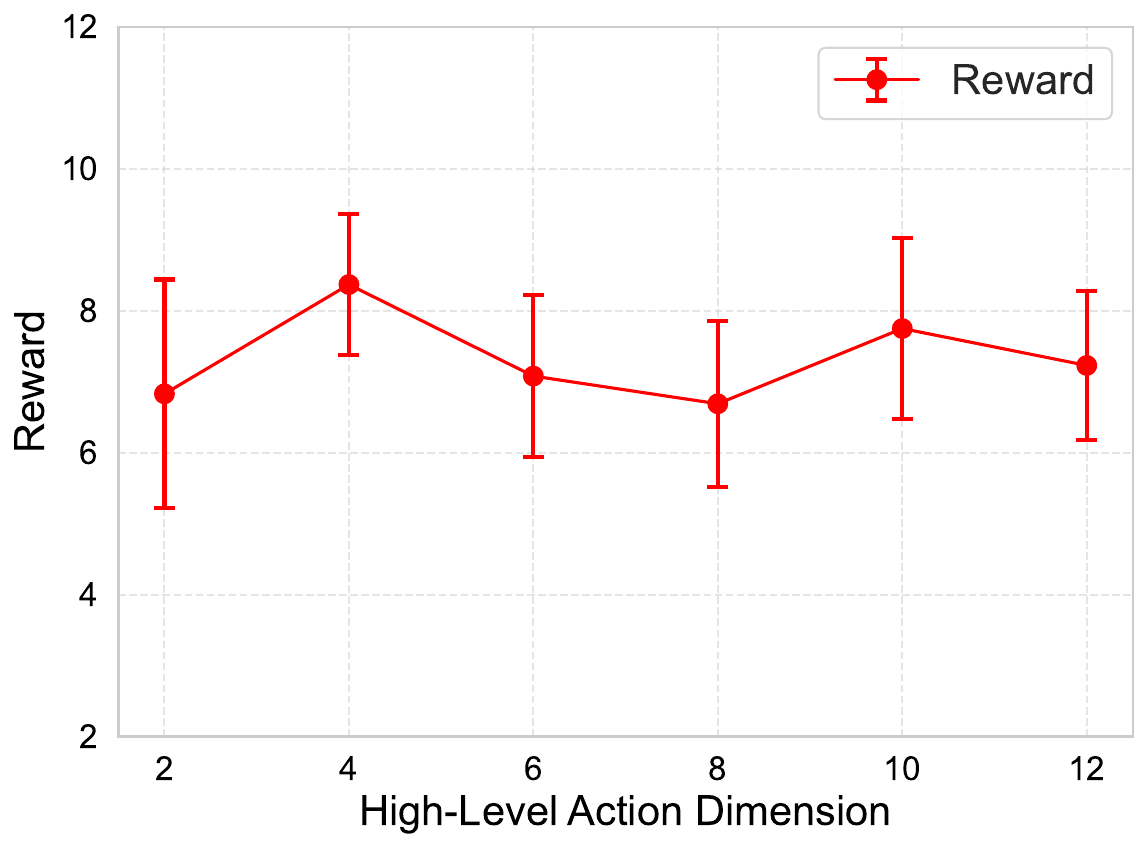}%
\label{high_action_dim}}
\hfil
\subfloat[]{\includegraphics[width=1.6in]{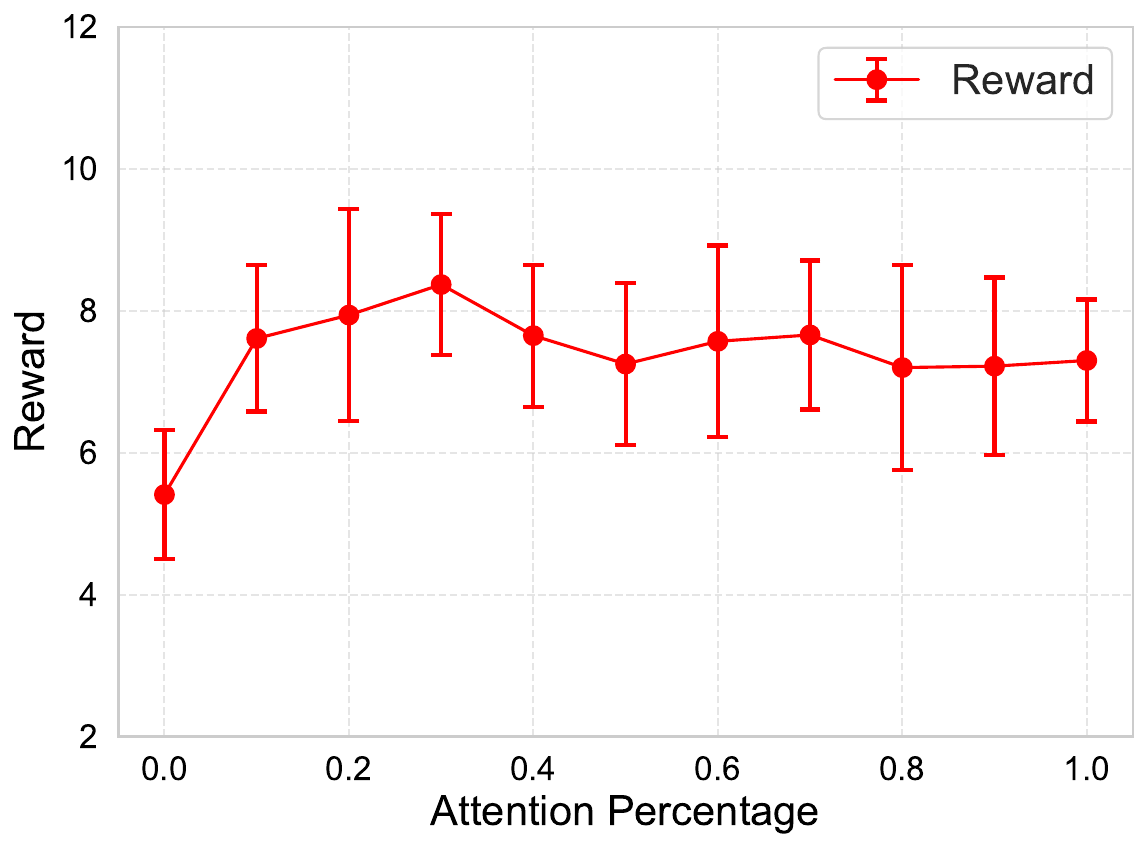}%
\label{att_per}}
\caption{Parameter Sensitivity Analysis Metrics: Converged reward, success rate, and collision rate. (a) Impact of high-level action dimension. (b) Impact of attention ratio.}
\label{fig:explore}
\end{figure}

\begin{itemize}
    \item \textbf{Analysis of the Impact of High-level Action Dimensions:} For the high-level action dimension, we fixed the attention ratio k = 0.3 and tested dimensions ranging from 2 to 12. The results, presented in Figure \ref{high_action_dim}, indicate that the performance remains relatively stable across different dimensions, demonstrating the robustness of the method to varying parameter settings.
    \item \textbf{Analysis of the Impact of Attention Ratio $k$:} For the attention ratio, we fixed the high-level action dimension at 4 and tested values from 0.0 to 1.0. As shown in Figure \ref{att_per}, a ratio of 0 leads to poor learning performance due to insufficient collision-related observations. Conversely, a extreme high ratio enlarges the observation space, making learning more difficult and slightly degrading performance. However, the performance remains stable across a broad range of moderate parameter values, highlighting the robustness of the method to parameter variations.
\end{itemize}

\subsubsection{Ablation Experiment}
To validate the effectiveness of our proposed method, we conduct an ablation study in a 20v20 setting, focusing on the agent-level attention module and the embedded mean-field aggregation module. We compare several variants:

\begin{itemize}
    \item \textbf{w/o Att-S}: Removing agent-level attention on states.
    \item \textbf{w/o Att-A}: Removing agent-level attention on embedded actions.
    \item \textbf{w/o Att-AS}: Removing the agent-level attention module entirely.
    \item \textbf{w/o EMF}: Removing embedded mean-field action aggregation, using only the agent-level attention module. We use MTMFAC, augmented with our attention module.
\end{itemize}

Experimental results in Table \ref{abla_exp} show that removing any module degrades performance. Eliminating attention on states causes significant drop, as the critics are overwhelmed with information, making it hard to extract useful insights. Removing attention on embedded high-level actions also results in a noticeable performance decline, indicating the attention mechanism’s role in capturing agent interactions and enhancing mean-field aggregation. The most significant performance decline occurs when the attention module is fully removed, highlighting its critical importance. Notably, when embedded mean-field action aggregation is removed but attention is retained, performance exceeds that of the original MTMFAC algorithm, demonstrating that the agent-level attention module can enhance various algorithms as a plug-and-play module. However, it still underperforms compared to EMFAC, which benefits from the learned embedded high-level actions, improving both convergence speed and final performance.
\begin{table}[t]
    \centering
    \caption{Ablation Experiment.}
    \label{tab:results}
    \setlength{\tabcolsep}{4pt} 
    \renewcommand{\arraystretch}{1.2} 
    \begin{tabular}{lccc}
        \toprule
        Algorithm & Converged Reward & Collision Rate (\%) & Success Rate (\%) \\
        \midrule
        EMFAC             & \textbf{8.37 $\pm$ 0.99}  & \textbf{2.17 $\pm$ 1.20}  & \textbf{65.75 $\pm$ 6.73}  \\
        w/o Att-S            & 4.54 $\pm$ 0.85  & 3.15 $\pm$ 1.29  & $45.13 \pm 4.68$  \\
        w/o Att-A            & 7.80 $\pm$ 1.67 & 2.22 $\pm$ 1.09 & 62.50 $\pm$ 9.90 \\
        w/o Att-AS           & 4.34 $\pm$ 1.06 & 3.12 $\pm$ 1.76 & 42.88 $\pm$ 6.65 \\
        w/o EMF   & 6.88 $\pm$ 2.09 & 2.74 $\pm$ 1.39 & 57.75 $\pm$ 12.30 \\
        MTMFAC   & 4.16 $\pm$ 1.47  & 3.17 $\pm$ 1.44 & 42.08 $\pm$ 8.42 \\
        \bottomrule
    \end{tabular}
    \label{abla_exp}
\end{table}

\subsubsection{Algorithm Runtime Analysis}
To thoroughly evaluate the computational efficiency of the proposed algorithm, we also analyze its overhead. Since all compared methods utilize the same actor network, their inference time remains identical, with the primary difference arising in the training of the Q-network. To capture this, we measure the total execution time of each algorithm, as shown in Fig. \ref{runtime}. The experiments are conducted on an NVIDIA RTX 3090 GPU and an Intel Core i9-12900K CPU. As the problem scale increases, the computational cost of all algorithms grows accordingly. However, the additional two modules of EMFAC require only a short training period before convergence, after which their parameters are frozen for inference. As a result, EMFAC incurs only a modest increase in computational overhead, while delivering significantly enhanced performance.
\begin{figure}[htp]
\centering
\includegraphics[width=5cm, height=3cm]{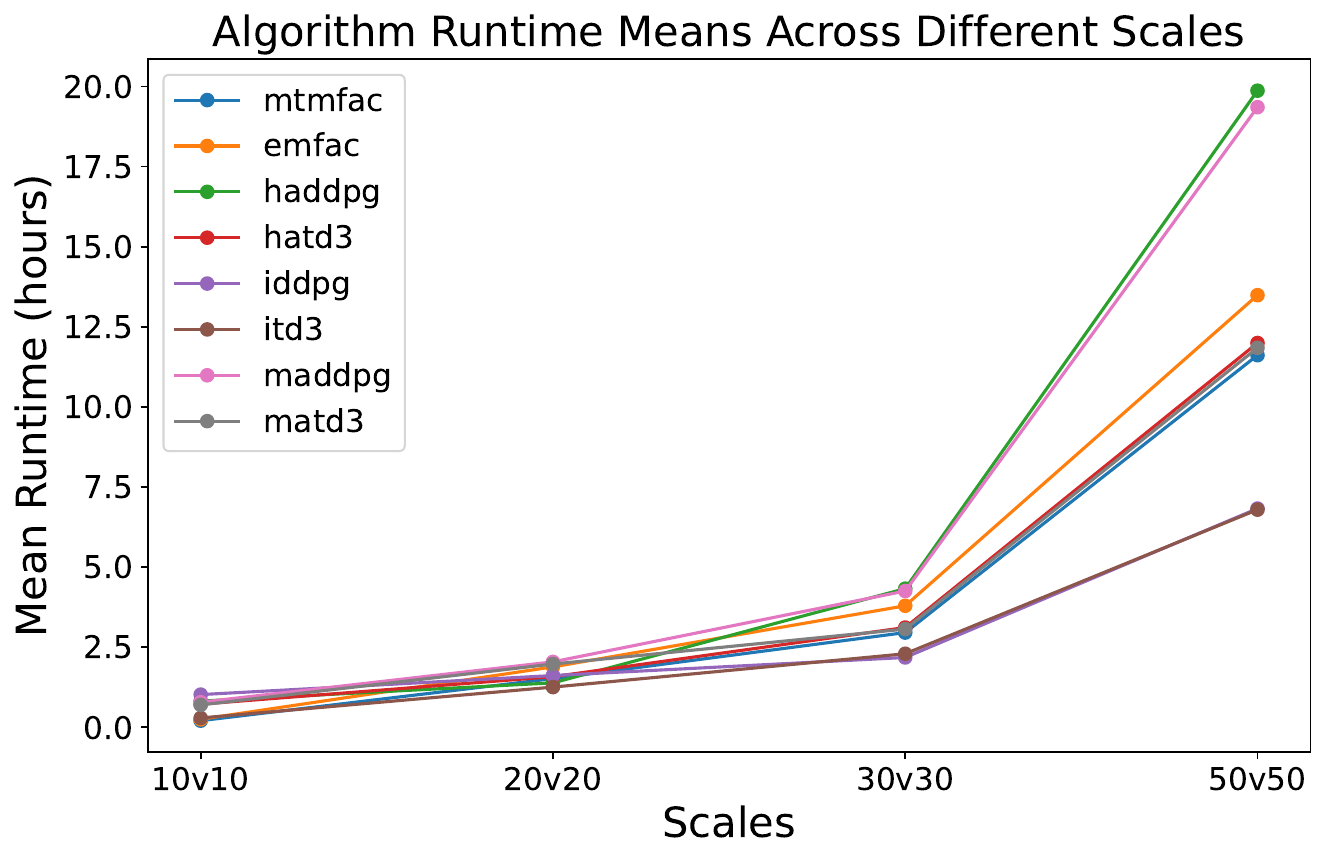}
\caption{Comparison of Algorithm Runtimes across Different Scales of Perimeter-defense Game.}
\label{runtime}
\end{figure}

\subsection{Real-World Experiments}
\begin{figure}[t]
\centering
\includegraphics[width=0.85\linewidth, height=4cm]{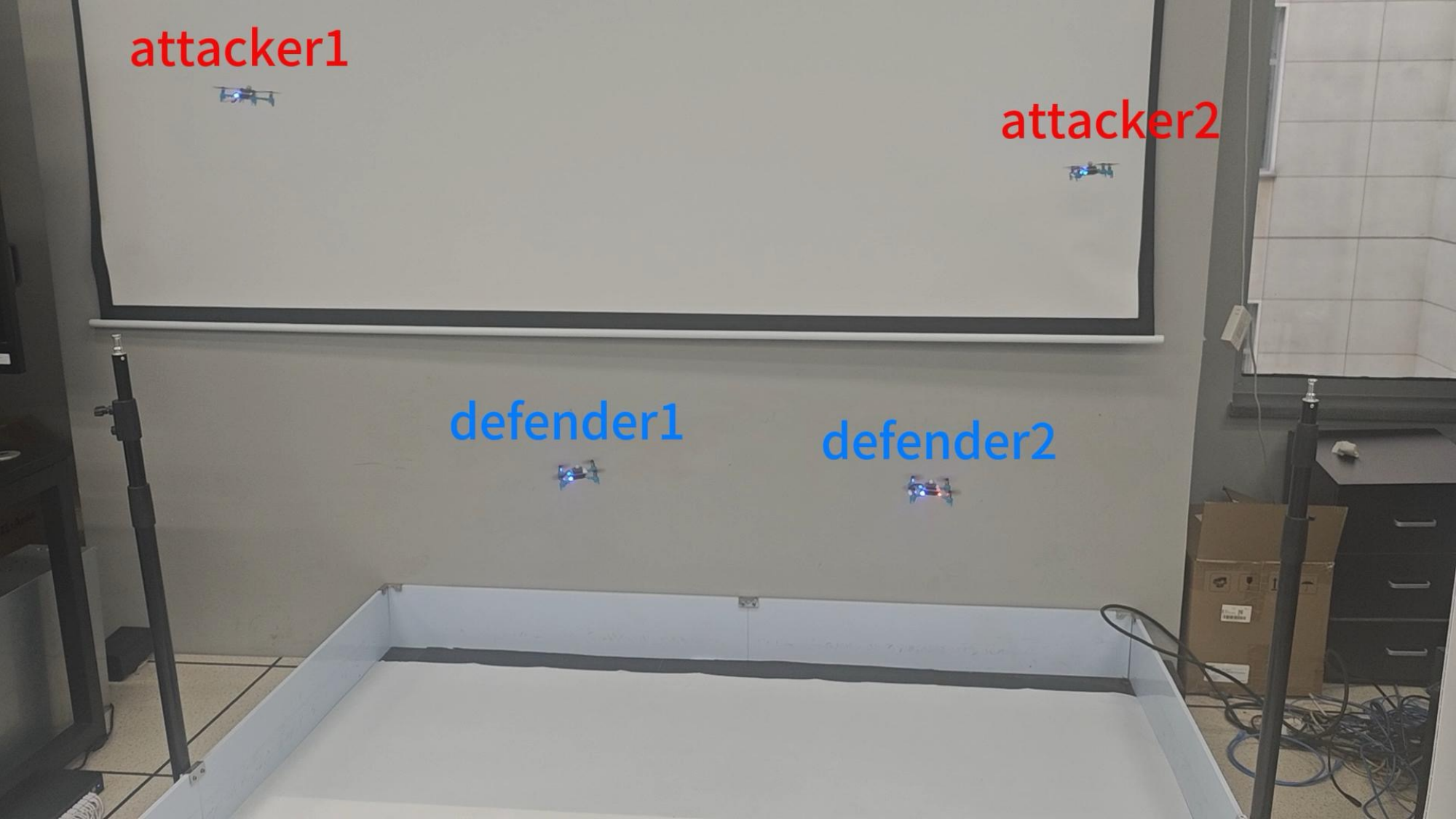}
\caption{Experimental setup for real-world validation using Crazyflie UAVs in 2v2 scenarios.}
\label{fig:real_experiment}
\end{figure}

To further validate the effectiveness of EMFAC, we conducted real-world experiments utilizing Crazyflie UAVs to simulate missile trajectories within a controlled indoor test environment, as depicted in Fig. \ref{fig:real_experiment}. The experimental arena, measuring 2 meters by 2 meters, was devoid of any obstacles to ensure a clear and unobstructed space for the UAVs. Small-scale scenarios were systematically executed in 2v2 configuration to assess the algorithm's performance under varying conditions. To accurately track the UAVs' movements, a centralized motion capture system was employed to obtain precise coordinates. Access to the experimental video \footnote{see supplementary materials.}, which captures the essence of these engagements. Table \ref{real_exp} presents the success rates and accuracy of the top five algorithms. EMFAC consistently achieves the highest accuracy while maintaining a lower collision rate, highlighting its effectiveness in real-world deployment.

\begin{table}[htbp]
\centering
\caption{Performance Comparison of Top-5 Algorithms in 2v2 Real-World Experiments}
\begin{tabular}{lccc}
\hline
\textbf{Algorithm} & \textbf{Collision Rate (\%)} & \textbf{Success Rate (\%)} \\
\hline
EMFAC & 0.22$\pm$0.39 & 98.78$\pm$2.90 \\
MTMFAC & 0.47$\pm$0.57 & 97.99$\pm$4.18 \\
IDDPG & 0.18$\pm$0.25 & 93.19$\pm$4.54 \\
ITD3 & 0.22$\pm$0.37 & 98.67$\pm$2.89 \\
MATD3 & 0.24$\pm$0.41 & 87.38$\pm$10.03 \\
\hline
\label{real_exp}
\end{tabular}
\end{table}

\section{Conclusion}
\label{sec:conclusion}
This paper introduces large-scale heterogeneous three-dimensional perimeter-defense game and utilizes motion dynamics, taking into account real-world factors such as wind perturbations. Based on the analysis of one-on-one case, Nash equilibrium strategies are derived and verified through simulation. For the many-versus-many heterogeneous scenarios, we propose the EMFAC method, which uses representation learning to learn high-level actions and employs reward learning to focus agent-level attention, accelerating the learning of agents and efficiently utilizing information in large-scale state-action spaces. The effectiveness of our method is validated through simulations and real-world experiments of varying scales, which can speed up training and significantly enhance convergence performance. Future work may include research on scenarios with an imbalance in the number of agents and varying task priority levels.

\section*{Acknowledgments}
This work was supported by the National Science and Technology Major Project (No. 2022ZD0117402)

{\appendix[PARAMETER SETTINGS]
For a fair comparison, all algorithms are implemented based on the HARL\cite{ref25, ref26} framework and share identical hyperparameter settings. The detailed parameter configurations are provided in Table \ref{param_setting}.

\begin{table}
\centering
\caption{Hyperparameters used in EMFAC.}
\begin{tabular}{cc}
\hline
Hyperparameters             & value                      \\ \hline
buffer size & 100000                         \\
batch size          & 256                      \\
gamma                       & 0.99                       \\
critic learning rate                  & 0.001                \\
actor learning rate                  & 0.0005                          \\
high action dim                 & 4 \\
attention rate             & 0.3    \\
optimizer                   & Adam                       \\
  expl noise           & 0.1                       \\
policy noise                & 0.2                          \\
share parameters      & False                 \\
activation func    & ReLU                       \\
hidden sizes   & [128, 128]                       \\ \hline \label{tablevalue}
\label{param_setting}
\end{tabular}
\end{table}

\section{Simple References}

\section{Biography Section}
\begin{IEEEbiography}[{\includegraphics[width=1in,height=1.25in,clip,keepaspectratio]{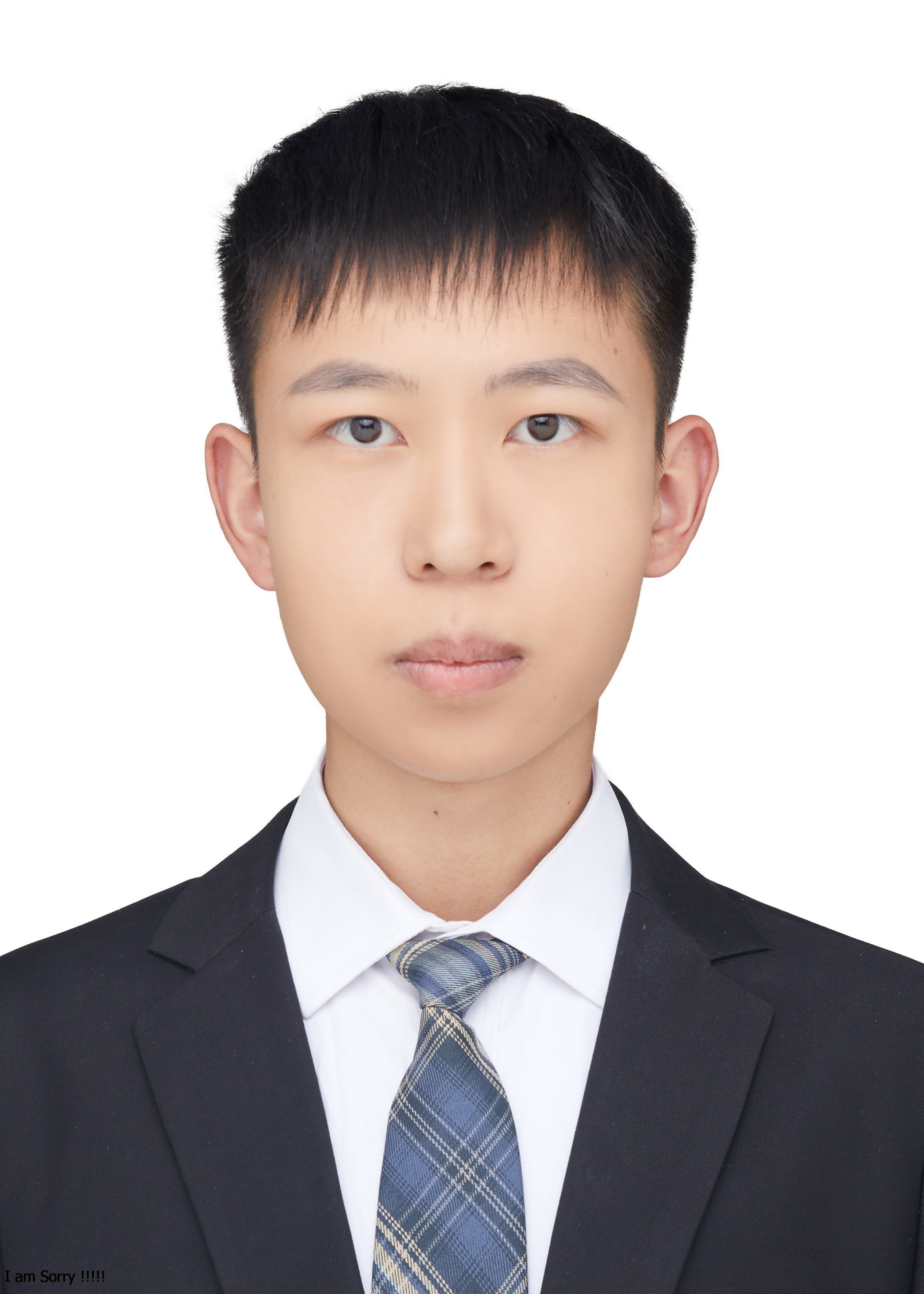}}]{Li Wang}
 received the B.S. degree from the School of Artificial Intelligence, Beihang University, in 2024. He is currently pursuing the Ph.D. degree at the same institution, under the supervision of Prof. Wenjun Wu. His research interests include multi-agent reinforcement learning and large language models.
\end{IEEEbiography}

\begin{IEEEbiography}
[{\includegraphics[width=1in,height=1.25in,clip,keepaspectratio]{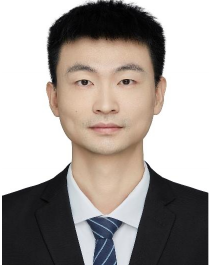}}]{Xin Yu} 
is a Ph.D. student at the School of Computer Science and Engineering at Beihang University, supervised by Prof. Wenjun Wu. His research focuses on MARL, aiming to enhance sample efficiency by integrating domain knowledge.
\end{IEEEbiography}

\begin{IEEEbiography}[{\includegraphics[width=1in,height=1.25in,clip,keepaspectratio]{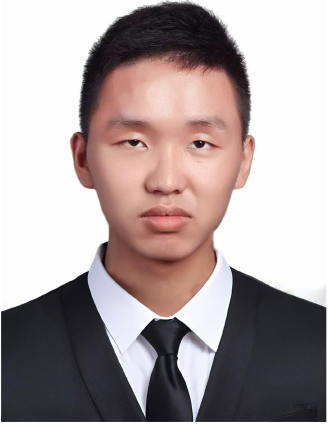}}]{Lv Xuxin}
 received the B.S. degree from the School of Artificial Intelligence, Beihang University, in 2024. He is currently pursuing the M.S. degree at the same institution, under the supervision of Prof. Wenjun Wu. His research interests include swarm intelligence and robotics.
\end{IEEEbiography}

\begin{IEEEbiography}[{\includegraphics[width=1in,height=1.25in,clip,keepaspectratio]{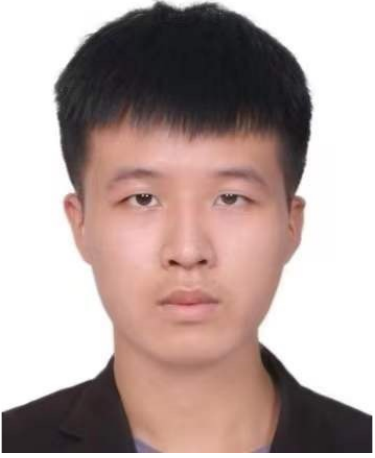}}]{Gangzheng Ai}
received the B.S. degree from School of Aeronautics and Astronautics, Sun Yat-Sen University, Guangzhou, China, in 2022. He is currently pursuing the M.S. degree with the Image Processing Center, School of Astronautics, Beihang University. His research interests include computer vision, Physics-Informed Neural Networks, 3D reconstruction, and deep-space exploration.
\end{IEEEbiography}

\begin{IEEEbiography}
[{\includegraphics[width=1in,height=1.25in,clip,keepaspectratio]{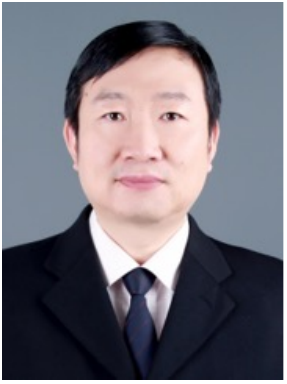}}]{Wenjun Wu}
 is a professor at the School of Artificial Intelligence, Beihang University. His research interests encompass swarm intelligence, large-scale online education, multi-agent reinforcement learning, and AI for science.
\end{IEEEbiography}

\vfill

\begin{thebibliography}{1}
\bibliographystyle{IEEEtran}
\bibitem{ref1}
D. Shishika and V. Kumar, ``A review of multi-agent perimeter defense games,'' in \textit{Proc. 11th Int. Conf. Decision and Game Theory for Security (GameSec)}, College Park, MD, USA, Oct. 28--30, 2020, pp. 472--485. Springer, 2020.

\bibitem{ref2}
D. Shishika and V. Kumar, ``Local-game decomposition for multiplayer perimeter-defense problem,'' in \textit{Proc. 2018 IEEE Conf. Decis. Control (CDC)}, 2018, pp. 2093--2100.

\bibitem{ref3}
Y. Alqudsi and G. El-Bayoumi, ``Modeling of target tracking system for homing missiles and air defense systems,'' \textit{INCAS Bulletin}, vol. 10, no. 2, 2018.

\bibitem{ref4}
M. A. Ma'Sum, M. K. Arrofi, G. Jati, F. Arifin, M. N. Kurniawan, P. Mursanto, and W. Jatmiko, ``Simulation of intelligent unmanned aerial vehicle (UAV) for military surveillance,'' in \textit{Proc. 2013 Int. Conf. Adv. Comput. Sci. Inf. Syst. (ICACSIS)}, 2013, pp. 161--166.

\bibitem{ref5}
J. Chai, W. Chen, Y. Zhu, Z.-X. Yao and D. Zhao, ``A hierarchical deep reinforcement learning framework for 6-DOF UCAV air-to-air combat,'' \textit{IEEE Transactions on Systems, Man, and Cybernetics: Systems}, vol. 53, no. 9, pp. 5417--5429, 2023, publisher: IEEE.

\bibitem{ref6}
K. S. Kappel, T. M. Cabreira, J. L. Marins, L. B. de Brisolara, and P. R. Ferreira, \textit{Strategies for patrolling missions with multiple UAVs}. Journal of Intelligent \& Robotic Systems, vol. 99, pp. 499--515, 2020.

\bibitem{ref7}
S. Bajaj, E. Torng, and S. D. Bopardikar, ``Competitive perimeter defense on a line,'' in \textit{Proc. 2021 American Control Conference (ACC)}, 2021, pp. 3196--3201.

\bibitem{ref8}
D. Shishika, J. Paulos, M. R. Dorothy, M. A. Hsieh, and V. Kumar, ``Team composition for perimeter defense with patrollers and defenders,'' in \textit{Proc. 2019 IEEE 58th Conf. Decis. Control (CDC)}, 2019, pp. 7325--7332.

\bibitem{ref11}
D. Shishika, J. Paulos, and V. Kumar, ``Cooperative team strategies for multi-player perimeter-defense games,'' \textit{IEEE Robotics and Automation Letters}, vol. 5, no. 2, pp. 2738--2745, 2020.

\bibitem{ref12}
D. Shishika and V. Kumar, ``Perimeter-defense game on arbitrary convex shapes,'' \textit{arXiv preprint arXiv:1909.03989}, 2019.

\bibitem{ref15}
S. Bajaj, S. D. Bopardikar, E. Torng, A. Von Moll, and D. W. Casbeer, ``Multivehicle perimeter defense in conical environments,'' \textit{IEEE Transactions on Robotics}, vol. 40, pp. 1439--1456, 2024.

\bibitem{ref9}
E. S. Lee, D. Shishika, and V. Kumar, ``Perimeter-defense game between aerial defender and ground intruder,'' in \textit{2020 59th IEEE Conference on Decision and Control (CDC)}, pp. 1530--1536, 2020.
```

\bibitem{ref10}
E. S. Lee, D. Shishika, G. Loianno, and V. Kumar, ``Defending a perimeter from a ground intruder using an aerial defender: Theory and practice,'' in \textit{Proc. 2021 IEEE Int. Symp. Safety, Security, Rescue Robotics (SSRR)}, 2021, pp. 184--189.

\bibitem{ref13}
E. S. Lee, L. Zhou, A. Ribeiro, and V. Kumar, ``Learning decentralized strategies for a perimeter defense game with graph neural networks,'' \textit{arXiv preprint arXiv:2211.01757}, 2022.

\bibitem{ref14}
E. S. Lee et al., ``Vision-based perimeter defense via multiview pose estimation,'' \textit{arXiv preprint arXiv:2209.12136}, 2022.

\bibitem{ref16}
A. Adler, O. Mickelin, R. K. Ramachandran, G. S. Sukhatme, and S. Karaman, ``The role of heterogeneity in autonomous perimeter defense problems,'' \textit{The International Journal of Robotics Research}, vol. 43, no. 9, pp. 1363--1381, 2024.

\bibitem{ref17}
I. Taub and T. Shima, ``Intercept angle missile guidance under time varying acceleration bounds,'' \textit{Journal of Guidance, Control, and Dynamics}, vol. 36, no. 3, pp. 686--699, 2013.

\bibitem{ref18}
A. {\c{S}}umnu and {\.I}. G{\"u}zelbey, ``The effects of different wing configurations on missile aerodynamics,'' \textit{Journal of Thermal Engineering}, vol. 9, no. 5, pp. 1260--1271, 2023.

\bibitem{ref19}
M. A. Ma'Sum et al., ``Dynamic Modeling, Guidance, and Control of Missiles,'' \textit{Middle East Technical University (METU)}, 2024.

\bibitem{ref20}
I. H. Jeong and H. G. Kim, ``Nonlinear Autopilot for Improving Guidance Performance of Dual-controlled Missiles With Lateral Thrust Regulation,'' in \textit{Proc. 39th Institute of Control, Robotics and Systems Conference}, 2024, pp. 129--130.

\bibitem{ref21}
N. Cho, S. Lee, J. Kim, Y. Kim, S. Park, and C. Song, ``Wind compensation framework for unpowered aircraft using online waypoint correction,'' \textit{IEEE Transactions on Aerospace and Electronic Systems}, vol. 56, no. 1, pp. 698--710, 2019.

\bibitem{ref22}
B. Fu, H. Qi, J. Xu, Y. Yang, S. Wang, and Q. Gao, ``Attitude control in ascent phase of missile considering actuator non-linearity and wind disturbance,'' \textit{Applied Sciences}, vol. 9, no. 23, p. 5113, 2019.

\bibitem{ref23}
S. Gu, J. Kuba, Y. Chen, Y. Du, L. Yang, A. Knoll, and Y. Yang, ``Safe multi-agent reinforcement learning for multi-robot control,'' \textit{Artificial Intelligence}, vol. 319, p. 103905, 2023.

\bibitem{ref24}
B. R. Kiran, I. Sobh, V. Talpaert, P. Mannion, A. A. Al Sallab, S. Yogamani, and P. P{\'e}rez, ``Deep reinforcement learning for autonomous driving: A survey,'' \textit{IEEE Transactions on Intelligent Transportation Systems}, vol. 23, no. 6, pp. 4909--4926, 2021.

\bibitem{ref125}
Y. Wang, L. Dong, and C. Sun, ``Cooperative control for multi-player pursuit-evasion games with reinforcement learning,'' \textit{Neurocomputing}, vol. 412, pp. 101--114, 2020.

\bibitem{ref126}
A. T. Bilgin and E. Kadioglu-Urtis, ``An approach to multi-agent pursuit evasion games using reinforcement learning,'' in \textit{2015 International Conference on Advanced Robotics (ICAR)}, pp. 164--169, 2015.

\bibitem{ref127}
R. Zhang, Q. Zong, X. Zhang, L. Dou, and B. Tian, ``Game of drones: Multi-UAV pursuit-evasion game with online motion planning by deep reinforcement learning,'' \textit{IEEE Transactions on Neural Networks and Learning Systems}, vol. 34, no. 10, pp. 7900--7909, 2022.

\bibitem{ref25}
J. Liu, Y. Zhong, S. Hu, H. Fu, Q. Fu, X. Chang, and Y. Yang, ``Maximum Entropy Heterogeneous-Agent Reinforcement Learning,'' in \textit{The Twelfth International Conference on Learning Representations}, 2024. [Online]. Available: <url id="cv4hd7h3om1t98gvtrp0" type="url" status="parsed" title="Maximum Entropy Heterogeneous-Agent Reinforcement Learning" wc="2250">https://openreview.net/forum?id=tmqOhBC4a5</url>

\bibitem{ref26}
Y. Zhong, J. Grudzien Kuba, X. Feng, S. Hu, J. Ji, and Y. Yang, ``Heterogeneous-Agent Reinforcement Learning,'' in \textit{J. Mach. Learn. Res.}, vol. 25, no. 32, pp. 1--67, 2024. [Online]. Available: \url{http://jmlr.org/papers/v25/23-0488.html}

\bibitem{ref27}
Y. Yang, R. Luo, M. Li, M. Zhou, W. Zhang, and J. Wang, ``Mean field multi-agent reinforcement learning,'' in \textit{Proc. Int. Conf. Mach. Learn.}, 2018, pp. 5571--5580.

\bibitem{ref28}
Y. Emami et al., ``Age of information minimization using multi-agent UAVs based on AI-enhanced mean field resource allocation,'' \textit{IEEE Transactions on Vehicular Technology}, 2024.

\bibitem{ref29}
Y. Xu et al., ``Joint Resource Allocation for V2X Communications With Multi-Type Mean-Field Reinforcement Learning,'' \textit{IEEE Transactions on Intelligent Transportation Systems}, 2024.

\bibitem{ref30}
D. Chen, Q. Qi, Z. Zhuang, J. Wang, J. Liao, and Z. Han, ``Mean Field Deep Reinforcement Learning for Fair and Efficient UAV Control,'' \textit{IEEE Internet of Things Journal}, vol. 8, no. 2, pp. 813--828, 2021. DOI: 10.1109/JIOT.2020.3008299.

\bibitem{ref31}
S. G. Subramanian, P. Poupart, M. E. Taylor, and N. Hegde, ``Multi type mean field reinforcement learning,'' \textit{arXiv preprint arXiv:2002.02513}, 2020.

\bibitem{ref32}
C. Yu, ``Hierarchical mean-field deep reinforcement learning for large-scale multiagent systems,'' in \textit{Proc. AAAI Conf. Artif. Intell.}, vol. 37,

\bibitem{ref33}
T. Wu, W. Li, B. Jin, W. Zhang, and X. Wang, ``Weighted mean-field multi-agent reinforcement learning via reward attribution decomposition,'' in \textit{Proc. Int. Conf. Database Syst. Adv. Appl.}, 2022, pp. 301--316. Springer.

\bibitem{ref34}
A. Vaswani et al., ``Attention Is All You Need,'' in \textit{Proc. 31st Int. Conf. Neural Information Processing Systems (NIPS)}, 2017, pp. 5998--6008.

\bibitem{ref35}
N. Botteghi et al., ``Unsupervised representation learning in deep reinforcement learning: A review,'' \textit{arXiv preprint arXiv:2208.14226}, 2022.

\bibitem{ref36}
M. Fraccaro, S. Kamronn, U. Paquet, and O. Winther, \textit{A disentangled recognition and nonlinear dynamics model for unsupervised learning}. Advances in Neural Information Processing Systems, vol. 30, 2017.

\bibitem{ref37}
D. Ha and J. Schmidhuber, \textit{Recurrent world models facilitate policy evolution}. Advances in Neural Information Processing Systems, vol. 31, 2018.

\bibitem{ref38}
D. Ha and J. Schmidhuber, \textit{World models}. arXiv preprint arXiv:1803.10122, 2018.

\bibitem{ref39}
E. Van der Pol, T. Kipf, F. A. Oliehoek, and M. Welling, \textit{Plannable approximations to MDP homomorphisms: Equivariance under actions}. arXiv preprint arXiv:2002.11963, 2020.

\bibitem{ref40}
G. Dulac-Arnold et al., \textit{Deep reinforcement learning in large discrete action spaces}. arXiv preprint arXiv:1512.07679, 2015.

\bibitem{ref41}
H. Song et al., ``MA2CL: Masked Attentive Contrastive Learning for Multi-Agent Reinforcement Learning,'' \textit{arXiv preprint arXiv:2306.02006}, 2023.

\bibitem{ref42}
Y. Chandak et al., ``Learning action representations for reinforcement learning,'' in \textit{Proc. Int. Conf. Mach. Learn.}, 2019, pp. 941--950.

\bibitem{ref43}
H. W. Kuhn, ``The Hungarian method for the assignment problem,'' \textit{Naval Research Logistics Quarterly}, vol. 2, pp. 83--97, 1955.

\bibitem{ref44}
He, S., Wang, W., Lin, D., \& Lei, H. (2017). Consensus-based two-stage salvo attack guidance. \textit{IEEE Transactions on Aerospace and Electronic Systems}, \textbf{54}(3), 1555--1566. \url{https://doi.org/10.1109/TAES.2017.2703360}.

\bibitem{ref45}
P. J. Huber, ``Robust estimation of a location parameter,'' \textit{Annals of Mathematical Statistics}, vol. 35, no. 1, pp. 73--101, 1964.

\bibitem{ref46}
T. P. Lillicrap, J. J. Hunt, A. Pritzel, N. Heess, T. Erez, Y. Tassa, D. Silver, and D. Wierstra, ``Continuous control with deep reinforcement learning,'' \textit{arXiv preprint arXiv:1509.02971}, 2015.

\bibitem{ref47}
S. Fujimoto, H. Hoof, and D. Meger, ``Addressing function approximation error in actor-critic methods,'' in \textit{Proc. Int. Conf. Mach. Learn.}, 2018, pp. 1587--1596.

\bibitem{ref48}
R. Lowe, Y. I. Wu, A. Tamar, J. Harb, P. Abbeel, and I. Mordatch, ``Multi-agent actor-critic for mixed cooperative-competitive environments,'' in \textit{Adv. Neural Inf. Process. Syst.}, vol. 30, 2017.

\end{thebibliography}
\end{document}